\documentclass[lettersize,journal]{IEEEtran}

\usepackage{times}
\usepackage{algorithm,algpseudocode}
\usepackage{linegoal} 
\usepackage{booktabs}
\usepackage{graphicx}
\usepackage{subcaption}
\usepackage{wrapfig}
\usepackage{cite}
\usepackage[colorlinks,
linkcolor=red,
anchorcolor=green,
citecolor=blue
]{hyperref}   
\usepackage{amsmath,amsfonts,amsthm}
\usepackage{bm}
\usepackage{url}
\usepackage{thmtools}
\usepackage{longtable}
\usepackage{multirow}

\newcommand{\gTf}{\mathcal{T}_{\sharp}}
\title{Learning to Solve the Quadratic Assignment Problem with Warm-Started MCMC Finetuning}

\usepackage{enumitem}
\author{Yicheng Pan\IEEEauthorrefmark{1}, Ruisong Zhou\IEEEauthorrefmark{1}, Haijun Zou\IEEEauthorrefmark{1}, Tianyou Li, Zaiwen Wen
	\IEEEcompsocitemizethanks{This research was supported in part by National Key Research and Development Program of China under the grant number 2024YFA1012900, the National Natural Science Foundation of China under the grant numbers 12331010 and
		12288101, and the Natural Science Foundation of Beijing, China under the grant number Z230002.
		\textit{(Corresponding author: Haijun Zou.)}}
	\IEEEcompsocitemizethanks{Yicheng Pan and Ruisong Zhou are with the School of Mathematical Science, Peking University, Beijing 100871, China (email: panyicheng@stu.pku.edu.cn; ruisongzhou@stu.pku.edu.cn). }
	\IEEEcompsocitemizethanks{Haijun Zou is with the State Key Laboratory of Mathematical Sciences, Institute of
		Computational Mathematics and Scientific/Engineering Computing, Academy
		of Mathematics and Systems Science, Chinese Academy of Sciences, Beijing
		100190, China (email: zouhaijun24@mails.ucas.ac.cn).}
	\IEEEcompsocitemizethanks{Tianyou Li and Zaiwen Wen are with the Beijing International Center for Mathematical Research, Peking University, Beijing 100871, China (email: tianyouli@stu.pku.edu.cn; wenzw@pku.edu.cn).}
    \IEEEcompsocitemizethanks{\IEEEauthorrefmark{1}These authors contributed equally to this work.}
}

\begin{document}

	\maketitle	
	\begin{abstract}
		The quadratic assignment problem (QAP) is a fundamental NP-hard task that poses significant challenges for both traditional heuristics and modern learning-based solvers. Existing QAP solvers still struggle to achieve consistently competitive performance across structurally diverse real-world instances. To bridge this performance gap, we propose PLMA, an innovative permutation learning framework. PLMA features an efficient warm-started MCMC finetuning procedure to enhance deployment-time performance, leveraging short Markov chains to anchor the adaptation to the promising regions previously explored. For rapid exploration via MCMC over the permutation space, we design an additive energy-based model (EBM) that enables an $O(1)$-time 2-swap Metropolis-Hastings sampling step.  Moreover, the neural network used to parameterize the EBM incorporates a scalable and flexible cross-graph attention mechanism to model interactions between facilities and locations in the QAP. Extensive experiments demonstrate that PLMA consistently outperforms state-of-the-art baselines across various benchmarks. In particular, PLMA achieves a near-zero average optimality gap on QAPLIB, exhibits remarkably superior robustness on the notoriously difficult Taixxeyy instances, and also serves as an effective QAP solver in bandwidth minimization.
	\end{abstract}
	\section{Introduction}
    \IEEEPARstart{I}{n} this paper, we consider the quadratic assignment problem (QAP), which seeks to assign $n$ facilities to $n$ locations to minimize the total cost, defined as the sum of products of flows between facilities and distances between their assigned locations. 
	 The QAP is one of the most fundamental combinatorial optimization (CO) problems with a wide range of real-world applications, such as facility location~\cite{elshafei1977facility_location},  graph matching~\cite{caetano2009graph_matching} and keyboard design~\cite{burkard1984keyboard_design}. Moreover, several other well-known CO problems can be formulated as QAPs. Typical examples include the traveling salesman problem, the maximal clique problem and the graph partitioning problem. The QAP also appears as subproblems in other CO problems, such as the bandwidth minimization (BM) problem. For more details, see~\cite{burkard1984keyboard_design, cela2013qap}
and references therein.
    
    Despite its broad applications, the QAP is not only NP-hard to solve but also NP-hard to approximate~\cite{sahni1976p}, rendering exact optimality practically intractable. In practice, exact solvers like branch-and-bound  and cutting plane methods struggle to solve instances of size larger than $n=20$ within a reasonable time budget~\cite{li1994lower_bound_qap}. By contrast, handcrafted heuristics can often find high-quality suboptimal solutions in acceptable computing time. Numerous meta-heuristics have been developed for the QAP, including  simulated annealing~\cite{burkard1984SA,connolly1990SA}, tabu search~\cite{Taillard1991,taillard1995search_qap} and memetic algorithms~\cite{merz2002maqap,BMA}. Among them, robust tabu search (Ro-TS)~\cite{Taillard1991} and the breakout memetic algorithm (BMA)~\cite{BMA} are widely recognized as particularly effective.  
    However, these methods are in general CPU-based, which limits their efficiency on large-scale instances, and they often require substantial instance-specific tuning.
	
	Machine learning offers a potential avenue to learn problem structure and reuse knowledge across instances while being able to leverage modern accelerators like GPUs  to improve the efficiency.  Recently, learning-based methods have achieved notable success on CO problems such as routing~\cite{kool2018attention, kwon2020POMO, wu2021learn-improvement-heuristics-VRP, li2025lmask} and scheduling problems~\cite{song2022flexible,jeon2023iclr,zhou2026wecan}. Much of this success leverages single-graph representations that align naturally with expressive neural architectures. The QAP, however, poses an additional challenge because each instance is specified by two coupled matrices, which can be viewed as two interacting graphs. Solving the QAP thus necessitates a sophisticated mechanism to fuse information across these two graphs. Existing neural methods typically approach this by either constructing an association graph with $n^2$ nodes~\cite{wang2021neural,liu2023RGM} or by concatenating facility and location embeddings via a complete permutation~\cite{tanlearning}. However, the former suffers from poor scalability, and the latter lacks generality as it assumes access to a full solution. 

    The effectiveness of both handcrafted heuristics and learning-based methods strongly depends on the instances to which they are applied. For handcrafted heuristics, performance is closely tied to the search landscape and neighborhood structure for test instances~\cite{BMA}.  For learning-based methods, the requirement of learning on in-distribution instances typically makes it hard to generalize well to target instances that can arise from arbitrary distributions. Indeed, real-world QAP benchmarks, such as the QAPLIB~\cite{burkard1997qaplib} and Taixxeyy instances~\cite{drezner2005taie}, comprise various instance families with significantly different structures. Therefore, no one algorithm demonstrates clear superiority in both efficiency and effectiveness across all types of QAP instances. Existing learning-based methods remain far from satisfactory, with average optimality gaps up to $37.82\%$ on QAPLIB~\cite{tanlearning}. At the same time, handcrafted heuristics often perform poorly and unstably on Taixxeyy instances, which are designed to be hard for local search algorithms~\cite{drezner2005taie}. 

These observations motivate a central question of this work. Can one develop a learning-based QAP solver that preserves the efficiency and transferability of learned priors while adapting effectively to test instances with previously unseen structure? To answer this question, we propose PLMA, a permutation learning framework built around an instance-adaptive Markov chain Monte Carlo (MCMC) finetuning mechanism. The key idea is to first learn general structural knowledge from a training distribution and then refine the pretrained model directly on  target test instances through the proposed efficient warm-started MCMC procedure. Our main contributions are summarized as follows.
	
	(1) We propose a two-stage learning paradigm for the QAP that unifies transferable pretraining with efficient deployment-time refinement.  The paradigm first pretrains a model on instances from a training distribution to acquire transferable structural priors, and then adapts it to  target test instances through a highly efficient batch-wise MCMC procedure. At the core of this finetuning stage is a warm-start mechanism that reuses previously obtained high-quality solutions to initialize customized short and locally interacting Markov chains, thereby steering the finetuning process toward promising regions of the search space.
	
	(2) We design an efficient and expressive probabilistic model for permutation learning in the QAP. The model is formulated as an energy-based model (EBM) tailored for MCMC sampling, whose additive structure enables $O(1)$-time evaluation of 2-swap proposals within a Metropolis-Hastings (MH) sampler. The neural network used to parameterize the model incorporates a simple yet effective cross-graph mechanism, which captures the coupled two-graph structure of the QAP in a scalable and flexible manner, without requiring a large association graph or a complete solution as input.
	
	(3) We demonstrate the state-of-the-art performance of PLMA through extensive experiments on various QAP benchmarks. Notably, PLMA achieves a near-zero average optimality gap on the canonical QAPLIB benchmark and exhibits clear superiority together with remarkable robustness on Taixxeyy instances. To the best of our knowledge, PLMA is the first learning-based QAP solver that surpasses leading heuristics in both solution quality and computational efficiency. In addition, we show that PLMA serves as an effective solver for QAP subproblems arising in the BM problem.

The remainder of this paper is structured as follows. A comprehensive review of the relevant literature is provided in Section~\ref{sec:related-work}. The mathematical formulations of the original QAP and the proposed learning problem are established in Section~\ref{sec:formulation}. Following this, the proposed PLMA framework is detailed in Section~\ref{sec:method}. The application of our solver to the BM problem is subsequently explored in Section~\ref{sec:application-BM}. Finally, extensive experimental validations across multiple benchmarks are presented in Section~\ref{sec:experiments} to substantiate the efficacy, efficiency, and robustness of PLMA.
	
    \section{Related Work}\label{sec:related-work}
	\subsection{Learning-Based Methods for the QAP}
	Learning-based methods for the QAP are relatively scarce in the literature. The majority of works focus on the application of QAP  to graph matching (GM) ~\cite{nowak2018revised, wang2019learning, wang2021neural, liu2023RGM}. To capture the quadratic property, a popular approach is to combine the two input graphs into an association graph, whose vertices are node pairs and edge weights are affinity scores. NGM \cite{wang2021neural} regards GM as a vertex classification task on the association graph and obtains the whole graph matching in one-shot. RGM \cite{liu2023RGM} constructs a graph matching by sequentially selecting vertices from the association graph.  Their reliance on the association graph, however, incurs significant computational cost. Towards directly solving the Koopmans-Beckmann's QAP and eliminating the need for an association graph, \cite{tanlearning} propose a solution-aware Transformer, which learns swap local search operators to iteratively refine initial complete solutions.

    \subsection{Deployment-Time Refinement}
	A variety of works to improve deployment-time performance of neural CO solvers have emerged in recent years~\cite{bello2016neural, li2023learning,Grinsztajn2023poppy,hottung2025polynet}. 
	A central paradigm is active search~\cite{bello2016neural}, where model parameters are finetuned on test instances. Variants include EAS~\cite{hottung2022efficient}, which updates only a subset of parameters or embeddings for efficiency, and DIMES~\cite{qiu2022dimes}, which employs meta-learning to provide favorable initialization for finetuning. These active search methods are generally built upon auto-regressive models, which factorize the overall probability into a product of conditional probabilities and construct solutions sequentially in an element-wise manner. However, this process requires restarting from scratch at each refinement iteration, hindering effective reuse of previously identified high-quality solutions. In contrast, our work introduces an EBM with MCMC sampling, enabling warm-starts from prior solutions and overcoming the restart limitation of auto-regressive active search.
	
	Although MCMC has been explored in MCPG \cite{chen2023MCPG}, our method differs in two key aspects. First, MCPG is designed for binary optimization problems. A direct extension to QAP would require modeling the permutation matrix with binary variables and enforcing permutation constraints via penalty terms. This substantially enlarges the search domain beyond the feasible set, making the generation of valid assignments inefficient. In contrast, our method operates directly over the permutation space, ensuring that every sampled solution is inherently feasible. Second, the policy in MCPG is parameterized as an instance-agnostic multivariate Bernoulli distribution, 
	hindering it from leveraging the transferable knowledge from pretraining. 
	Our policy is parameterized by an instance-conditioned neural network, allowing the learning of general structural knowledge.
	    
	\subsection{MCMC Methods for CO}
    MCMC methods offer a general probabilistic framework for CO, in which Markov chains are used to sample solutions from distributions that favor lower-cost states. Simulated annealing is a typical representative of this framework, which explores the solution space through MH transitions under a sequence of objective-induced EBMs with annealing temperatures. However, the random walk behavior of MH may lead to a vast number of sampling steps. Addressing this, some recent works have replaced the MH sampler with gradient-based samplers derived from Langevin dynamics (LD). Sun et al.~\cite{iSCO} first adapt LD techniques to tackle CO problems. Feng \& Yang ~\cite{RLSA} further propose a regularized approach on discrete LD, effectively escaping local minima. Our work builds on this broader line of research, yet takes a different and less explored perspective, namely, leveraging the warm-start property of MCMC to make multi-round finetuning on target instances more effective. Accordingly, rather than focusing on the design of state transitions within a Markov chain, we develop a short-chain sampling mechanism with local interactions across chains to keep successive rounds concentrated near previously identified high-quality solutions.

	\section{Problem Formulations}\label{sec:formulation}
	\subsection{Quadratic Assignment Problem}
	In this paper, we consider the Koopmans-Beckmann's QAP, which involves the optimal assignment of $n$ facilities to $n$ locations. A QAP instance $\mathcal{P}$ is specified by two $n\times n$ matrices with real elements $\displaystyle \mathbf{F}=(F_{ij}), \displaystyle \mathbf{D}=(D_{kl})$, where $F_{ij}$ is the flow  between facilities $i$ and $j$ and $D_{kl}$ is the  distance between locations $k$ and $l$. For any positive integer $n$, we denote the set $\{1,\dots,n\}$ by $[n]$. Let $\Pi_n$ represent the set of all permutations over $[n]$. A permutation $\pi\in\Pi_n$ maps each facility $i$ to location $\pi(i)$. The QAP can be formulated as 
	\begin{equation}
		\label{eq:qap-pi}
		\min_{\pi\in \Pi_n} f(\pi;\displaystyle \mathcal{P})\;:=\;\sum_{i=1}^n\sum_{j=1}^n F_{ij}D_{\pi(i)\pi(j)}.
	\end{equation}
	
	Alternatively, the QAP can be expressed equivalently in matrix form
	\begin{equation*}
		\label{eq:qap-matrix}
		\begin{aligned}
			\min_{X\in\{0,1\}^{n\times n}} \quad &\langle F,XDX^{\rm T}\rangle,\quad \mathrm{s.t.} \quad  X\bm{1}=\bm{1},\,  
			X^{\rm T}\bm{1} = \bm{1}.
		\end{aligned}
	\end{equation*}
	Here, the linear constraints restrict the binary matrix $X$ to the set of permutation matrices. For each $\pi\in\Pi_n$, its corresponding permutation matrix $X_{\pi}$ is defined by $(X_{\pi})_{i,\pi(i)}=1$ and zero otherwise. Then the equivalence to \eqref{eq:qap-pi} is established by the identity $f(\pi,\mathcal{P})=\langle F,X_\pi D X^{\rm T}_\pi\rangle$.

	\subsection{Probabilistic Modeling Perspective}
	Let $\Delta(\Pi_n)$ denote the probabilistic simplex over the permutation space $\Pi_n$, defined by
	\begin{equation*}
		\Delta(\Pi_n):=\{p:\Pi_n\to \mathbb R_+ | \; \sum_{\pi\in\Pi_n}p(\pi)=1\}.
	\end{equation*}
	We then consider the following optimization problem:
	\begin{equation}\label{eq:prob-form}
			\min_{p\in\Delta(\Pi_n)}\quad  \mathbb{E}_{\pi\sim p}\left[f(\pi;\mathcal{P})\right]\;:=\;\sum_{\pi\in\Pi_n}f(\pi;\mathcal{P})p(\pi\mid \mathcal{P}).
	\end{equation}
	This probabilistic formulation is provably equivalent to the original deterministic model (\ref{eq:qap-pi}) of the QAP, as formalized in Theorem~\ref{thm:eq}.
	\begin{restatable}{theorem}{exactequivalence}
		\label{thm:eq}
		Let $v^\star=\min\limits_{\pi\in\Pi_n} f(\pi;\mathcal{P})$ and $w^\star=\min\limits_{p\in\Delta(\Pi_n)}\mathbb{E}_{p}[f(\pi;\mathcal{P})]$. Then $w^\star=v^\star$. Moreover, the minimizers of (\ref{eq:prob-form}) are precisely the distributions supported on the optimal solution set of (\ref{eq:qap-pi}).
	\end{restatable}
	\begin{proof}
		For any $p\in\Delta(\Pi_n)$,
		$\mathbb{E}_{p}[f]\ge \min_{\pi} f(\pi)=v^\star$, hence $w^\star\ge v^\star$.
		If $\pi^\star\in\arg\min_{\pi} f(\pi)$ and $p=\delta_{\pi^\star}$, then $\mathbb{E}_{p}[f]=f(\pi^\star)=v^\star$, hence $w^\star\le v^\star$. Thus $w^\star=v^\star$. Linearity on the simplex implies optima occur at extreme points or their convex combinations over the minimizer set.
	\end{proof}
	
	It is computationally intractable to optimize directly over the entire probability simplex due to the factorial growth of its dimension. In practice, we would replace the whole probabilistic space with a family of probabilistic models $\{p_{\theta}(\cdot\mid\mathcal{P})\mid \theta\in\mathbb{R}^d\}$ parametrized by a neural network with parameters $\theta$.  To allow learning transferable distribution knowledge, this network is instance-conditioned, which takes a problem instance $\mathcal{P}$ as input and generates a conditional probabilistic distribution $p_{\theta}(\cdot\mid\mathcal{P})$.  Training this network aims to minimize the expected cost over a distribution of instances $\Gamma$, leading to the learning problem
	\begin{equation*}
    \min_{\theta\in\mathbb{R}^d}\quad \mathbb{E}_{\mathcal{P}\sim \Gamma}\mathbb{E}_{\pi\sim p_{\theta}(\cdot\mid \mathcal{P})}\left[f(\pi;\mathcal{P})\right].
	\end{equation*}
	
	Directly minimizing this expectation is difficult because the objective landscape is typically rugged and riddled with poor local minima. To address this, we push the entire distribution $p_{\theta}$ forward to yield a flatter distribution. This is realized by a local improvement map $\mathcal{T}:\Pi_n\to\Pi_n$ based on the 2-swap neighborhood. This map is a common improvement strategy for many combinatorial optimization problems, including the QAP. Denote 2-swap action as an unordered pair $a=(r,s)$ with $r \neq s$, which transforms a permutation $\pi$ into a new permutation $\pi'$ by exchanging the assignments at positions $r$ and $s$:
	\begin{align*}
		\pi'(i) &= \pi(i), \quad \forall i \notin \{r,s\}, \\
		\pi'(r) &= \pi(s), \\
		\pi'(s) &= \pi(r).
	\end{align*}
	A key advantage of the 2-swap is that the change in the objective function, $\Delta(\pi, a) = f(\pi') - f(\pi)$, can be calculated efficiently without re-evaluating the entire summation. The update can be computed in $O(n)$ time as follows:
	\begin{align*}
		\Delta(\pi, a)
		= & \;(F_{rr} - F_{ss})(D_{\pi(s)\pi(s)} - D_{\pi(r)\pi(r)}) \\
		& +\;(F_{rs} - F_{sr})(D_{\pi(s)\pi(r)} - D_{\pi(r)\pi(s)}) \\
		& +\;\sum_{k \notin \{r,s\}} (F_{rk} - F_{sk})(D_{\pi(s)\pi(k)} - D_{\pi(r)\pi(k)})\\
		&+\;\sum_{k \notin \{r,s\}}  (F_{kr} - F_{ks})(D_{\pi(k)\pi(s)} - D_{\pi(k)\pi(r)}).
	\end{align*}
	The local improvement map operates by sampling a batch of potential 2-swap actions in parallel at each iteration. It then identifies the single best action from this batch—the one that yields the greatest decrease in cost. If this best action results in an improved solution, the permutation $\pi$ is updated.
	This local improvement map is detailed in Algorithm \ref{alg:local_search}.
	
	\begin{algorithm}[!htbp]
		\caption{Local improvement map for QAP}
		\label{alg:local_search}
		\begin{algorithmic}[1]
			\State \textbf{Input:} Initial solution $\pi$, number of local search iterations $T_{\text{LS}}$, number of 2-swap candidates per iteration $K_{\text{LS}}$.
			
			\For{$t=1, \dots, T_{\text{LS}}$}
			\State Sample a set of $K_{\text{LS}}$ random 2-swap actions in parallel: $A = \{a_1, a_2, \dots, a_{K_{\text{LS}}}\}$.
			\State Calculate $\Delta_i = \Delta(\pi, a_i)$ for all $a_i \in A$.
			\State Find the best action: $a^* = \arg\min_{a_i \in A} \Delta_i$.
			\State Let $\Delta^* = \Delta(\pi, a^*)$.
			\State \textbf{if} $\Delta^* < 0$ \textbf{then}
			\State Update $\pi \leftarrow \pi'$, where $\pi'$ is the result of applying action $a^*$ to $\pi$.
			\State \textbf{end if}
			\EndFor
			\State \textbf{Return:} The refined solution $\pi$.
		\end{algorithmic}
	\end{algorithm}
	
	Given the local improvement map $\mathcal T$, the pushforward distribution $\gTf p_{\theta}$ is then defined as $ (\gTf p_{\theta})(\sigma):=\sum_{\pi\in\mathcal{T}^{-1}(\sigma)}p_{\theta}(\pi)$, which aggregates the probability mass of all permutations $\pi$ that are mapped by $\mathcal T$ to the same improved solution $\sigma=\mathcal T(\pi)$. In this way, it creates a flatter distribution supported on the set of improved solutions, reducing the risk of getting trapped in poor local optima.  After applying the pushforward transformation, the learning problem becomes \begin{equation}
		\label{eq:pushforward-obj}
		\begin{aligned}
			\min_{\theta\in\mathbb{R}^d}\quad  \mathcal{L}(\theta):=&\mathbb E_{\mathcal{P} \sim\Gamma} \mathbb{E}_{\sigma\sim \gTf p_{\theta}(\cdot\mid \mathcal{P})}\left[f(\sigma;\mathcal{P})\right]\\
			=& \mathbb{E}_{\mathcal{P}\sim \Gamma}\mathbb{E}_{\pi\sim p_{\theta}(\cdot\mid\mathcal{P})}\left[f(\mathcal{T}(\pi);\mathcal{P})\right].
		\end{aligned}
	\end{equation}
	The second equality above, which follows from the change-of-variables formula for pushforward measures, is pivotal to our method. It shows that the underlying sampler $p_{\theta}$ is trained to place probability mass on regions that yield high-quality solutions after local improvement. At the same time, it provides a practical way to optimize $\mathcal L(\theta)$: one can sample from the original, easy-to-sample distribution $p_\theta$ and then evaluate the improved cost $f(\mathcal T(\pi);\mathcal P)$. 
	
	\section{PLMA Framework} \label{sec:method}

    In this section, we present the details of PLMA, our proposed learning framework for the QAP. At a high level, PLMA approaches QAP solving by learning an optimal probabilistic model over the permutation space. Specifically, the model is formulated as an EBM whose energy is induced by a heatmap and admits an additive structure, enabling efficient solution sampling via MCMC. The heatmap is generated by a scalable cross-graph attention network that encodes QAP instance features and captures latent relationships between disconnected nodes. The resulting model is pretrained on a problem distribution under the pushforward learning formulation introduced in Section~\ref{sec:formulation}, so as to learn transferable structural knowledge of the QAP. Finally, it is refined via a warm-started MCMC finetuning stage with short Markov chains to improve adaptation to target instances whose structures may differ substantially from those seen during training.

	\subsection{Learning with an Efficient Energy-Based Model}
	
	Given a heatmap from the network, existing works~\cite{wang2019learning, wang2021neural} often project it into permutations using the non-differentiable Hungarian algorithm, which limits diversity and optimization. To circumvent these issues, we introduce a subtly designed EBM that is computationally tailored and supports efficient exploration of the permutation space via Metropolis–Hastings (MH) sampling. The relationship between a permutation matrix $X_{\pi}$ and a heatmap $\phi=\phi(\theta,\mathcal P)\in\mathbb{R}^{n\times n}$ can be modeled using an EBM, derived from a distance metric as follows:
	\begin{equation}
		\label{eq:ebm}
		\begin{aligned}
			p_\theta(\pi\mid \mathcal{P}) &\propto \exp\left(-\frac{1}{2}\|X_\pi-\phi\|^2_{F}\right)\\
			&= \exp\left(\langle X_\pi, \phi\rangle-\frac{1}{2}\|X_\pi\|^2_{F}-\frac{1}{2}\|\phi\|^2_{F}\right)\\
			& \propto \exp \left(\langle X_\pi, \phi\rangle\right)=:\exp\left(\Phi_\theta(\pi)\right).
		\end{aligned}
	\end{equation}
	Here the second proportionality holds because $\|X_\pi\|^2_{F}$ is a constant value of $n$ and $\|\phi\|^2_{F}$ is independent of the random variable $\pi$.  These terms are thus absorbed into the normalization constant. The resulting additive score, defined as $\Phi_\theta(\pi):=\langle X_\pi, \phi\rangle=\sum_{i=1}^n\phi_{i,\pi(i)}$, is the key to our approach. This additive design is deliberate, as it allows efficient evaluation of score differences when exploring the space of permutations.
	
	To sample from this distribution, which is known only up to the intractable $Z_\theta$, we employ the MH algorithm with proposal kernel $\kappa(\cdot\mid\pi)$ and acceptance probability
	\begin{equation}
		\label{eq:mh-general}
		\alpha(\pi\to\pi') \;=\; \min\!\left\{1,\,
		\frac{p_\theta(\pi')\,\kappa(\pi\mid\pi')}
		{p_\theta(\pi)\,\kappa(\pi'\mid\pi)}\right\}.
	\end{equation}
	We adopt a symmetric 2-swap proposal that uniformly selects two positions $a\neq b$ and swaps $\pi(a)$ with $\pi(b)$. The proposal's symmetry $\kappa(\pi\mid\pi')=\kappa(\pi'\mid\pi)$ simplifies the Hastings ratio in (\ref{eq:mh-general}) to
	\begin{equation*}
		\label{eq:mh-symmetric}
		\exp(\Phi_\theta(\pi')-\Phi_\theta(\pi))=\exp\!\big(\phi_{a,\pi(b)}+\phi_{b,\pi(a)}-\phi_{a,\pi(a)}-\phi_{b,\pi(b)}\big),
	\end{equation*}
	which can be computed in constant time. Thus, each update has $O(1)$ cost, and the stationary distribution of the chain remains $p_\theta(\pi)$ with the normalization constant canceled. This leads to a remarkably efficient MCMC sampling scheme.
	
	Although the exact log-likelihood $\log p_\theta(\pi)$ is not available due to the intractable $Z_\theta$, we can still estimate its gradient without computing  $Z_\theta$. In particular, an unbiased and consistent Monte Carlo estimator enables the use of policy-gradient optimization. This is formalized in Theorem~\ref{thm:unbiased-gradient-estimator}, with the details of proof provided in supplementary materials.

	\begin{restatable}{theorem}{gradient}\label{thm:unbiased-gradient-estimator}
		Let $\pi_1,...,\pi_N$ be i.i.d. samples from $p_\theta$ in (\ref{eq:ebm}). Then for any function $g:\Pi_n\to\mathbb{R}$, the gradient of the expectation $\mathbb{E}_{\pi\sim p_\theta}\left[g(\pi)\right]$ admits an unbiased and consistent estimator 
		\begin{equation}
			\label{eq:gd-est}
			\hat{G}_N \;:=\frac{1}{N-1}\sum_{i=1}^N \left(g(\pi_i)-\frac{1}{N}\sum_{i=1}^Ng(\pi_i)\right)\nabla_\theta\Phi_\theta(\pi_i).
		\end{equation}
	\end{restatable}
	
	\begin{proof}[Proof sketch]
		Recall $p_\theta(\pi)\propto \exp(\Phi_\theta(\pi))$ over the finite set $\Pi_n$.
		Using the score-function identity,
		\begin{align}
			\nabla_\theta \mathbb E_{p_\theta}[g(\pi)]
			&=\mathbb E_{p_\theta}\!\big[g(\pi)\nabla_\theta\log p_\theta(\pi)\big]. \label{eq:sketch-score}
		\end{align}
		Moreover, $\log p_\theta(\pi)=\Phi_\theta(\pi)-\log Z_\theta$ and
		\begin{align*}
			\nabla_\theta \log Z_\theta
			=\frac{1}{Z_\theta}\sum_{\hat\pi}\exp(\Phi_\theta(\hat\pi))\nabla_\theta\Phi_\theta(\hat\pi)
			=\mathbb E_{p_\theta}[\nabla_\theta\Phi_\theta(\pi)].
		\end{align*}
		which turns \eqref{eq:sketch-score} into the covariance form
		\begin{align}
			\nabla_\theta \mathbb E_{p_\theta}[g(\pi)]
			=\mathbb E_{p_\theta}\!\Big[(g(\pi)-\mathbb E_{p_\theta}[g])\,\nabla_\theta\Phi_\theta(\pi)\Big].
			\label{eq:sketch-cov}
		\end{align}
		
		Let $\pi_1,\ldots,\pi_N\overset{\text{i.i.d.}}{\sim}p_\theta$ and denote
		$\bar g=\frac1N\sum_i g(\pi_i)$,
		$\overline{\nabla\Phi}=\frac1N\sum_i \nabla_\theta\Phi_\theta(\pi_i)$, and
		$\overline{g\nabla\Phi}=\frac1N\sum_i g(\pi_i)\nabla_\theta\Phi_\theta(\pi_i)$.
		Then the proposed estimator can be rewritten as
		\begin{align}
			\hat G_N=\frac{N}{N-1}\Big(\overline{g\nabla\Phi}-\bar g\,\overline{\nabla\Phi}\Big).
			\label{eq:sketch-est}
		\end{align}
		(Unbiasedness) By independence,
		\begin{align}
			\mathbb E[\bar g\,\overline{\nabla\Phi}]
			&=\frac{1}{N^2}\sum_{i=1}^N \mathbb E[g_i\nabla\Phi_i]
			+\frac{1}{N^2}\!\!\sum_{i\neq j}\!\!\mathbb E[g_i]\mathbb E[\nabla\Phi_j]  \notag\\
			&=\frac{1}{N}\mathbb E[g\nabla\Phi]+\frac{N-1}{N}\mathbb E[g]\mathbb E[\nabla\Phi], \label{eq:sketch-prodmean}
		\end{align}
		while $\mathbb E[\overline{g\nabla\Phi}]=\mathbb E[g\nabla\Phi]$.
		Plugging \eqref{eq:sketch-prodmean} into \eqref{eq:sketch-est} yields
		$\mathbb E[\hat G_N]=\mathbb E[g\nabla\Phi]-\mathbb E[g]\mathbb E[\nabla\Phi]$,
		which matches \eqref{eq:sketch-cov}.
		
		(Consistency) Since $\Pi_n$ is finite, the involved random variables are bounded;
		thus by the SLLN, $\bar g$, $\overline{\nabla\Phi}$ and $\overline{g\nabla\Phi}$
		converge a.s. to their expectations, and $\frac{N}{N-1}\to 1$.
		Therefore $\hat G_N \to \nabla_\theta \mathbb E_{p_\theta}[g(\pi)]$ almost surely.
	\end{proof}
	
	\subsection{Cross-Graph Attention Network}
	
	A QAP instance can be represented as two weighted graphs, one for locations and one for facilities. To maintain scalability, we encode each graph separately with edge-based GNNs, embedding the $n^2$ pairwise information into edges and aggregating it into $n$ node embeddings. A cross-graph attention module then captures interactions between the two graphs, and the resulting embeddings are combined through a dot product and Sinkhorn normalization to form a heatmap $\phi(\theta,\mathcal P)$. 
	
 Initial node features are set by learnable parameters $\mathbf{h}_{\text{ini}}\in \mathbb R^{d_{in}}$ as $\displaystyle \mathbf{X}^0 = \bm{1}\mathbf{h}_{\text{ini}}^{\rm T}\in \mathbb{R}^{2n\times d_{in}}$, which are then projected to the initial embeddings $\displaystyle \mathbf{H}^{(0)}=\displaystyle \mathbf{X}^0\displaystyle \mathbf{W}_{\mathrm{proj}}\in\mathbb{R}^{2n\times d}$. These embeddings are subsequently updated through $l_1$ message passing layers. Each layer updates the node embeddings $\displaystyle \mathbf{H}$ based on the input graph $\displaystyle \mathbf{D}$ and $\displaystyle \mathbf{F}$ by combining a standard graph convolution network (GCN) operation with a direct algebraic operation on the weight matrix:
	\begin{equation}
		\label{eq:gcn}
		\begin{bmatrix}
			{\displaystyle \mathbf{H}}_D^{(l)}\\  {\displaystyle \mathbf{H}}_F^{(l)}
		\end{bmatrix} = \sigma\left[\begin{bmatrix}
			\displaystyle \mathbf{D} - \bar {\displaystyle \mathbf{D}}& \\
			& \displaystyle \mathbf{F} - \bar {\displaystyle \mathbf{F}}
		\end{bmatrix} \begin{bmatrix}
			\displaystyle \mathbf{H}_D^{(l-1)} \displaystyle \mathbf{W}_{\text{alg,D}}^l \\
			\displaystyle \mathbf{H}_F^{(l-1)} \displaystyle \mathbf{W}_{\text{alg,F}}^l
		\end{bmatrix}\right].
	\end{equation}
	Here, $\bar{\mathbf{D}} = \bar d \bm{1}\bm{1}^{\rm T}$ and $\bar{\mathbf{F}} = \bar f \bm{1}\bm{1}^{\rm T}$, and $\bar {d}$, $\bar{f}$ are the averages of elements of $\mathbf{D}$ and $\mathbf{F}$. The layers are stacked with residual connections and layer normalization. By incorporating the edge information into the weighted convolutions, the original $2 n^2$ pair information is implicitly transferred into the $2n$ node embeddings while retaining the scalability. 
	
To break the disconnectedness of two graphs, a cross-attention module updates the embeddings. This module is analogous to a standard transformer encoder block. It allows each set of embeddings to be updated based on information from the other, producing contextually aware final embeddings. The initial embeddings are set as the GCN output: $\tilde{\mathbf{H}}_D^{(0)} = {\mathbf{H}}_D^{(l_1)}, \tilde{\mathbf{H}}_F^{(0)} = {\mathbf{H}}_F^{(l_1)}$. In the $r$-th of the total $l_2$ layers, the attention scores $e^{r}_{D,ij}$ and $e^{r}_{F,ij}$ are calculated:
	\begin{equation*}
		\begin{aligned}
			&a_{D,ij}^{(r)} = (\tilde{\mathbf{h}}_{D,i}^{r-1}\mathbf{W}_{D,q})(\tilde{\mathbf{h}}_{F,j}^{r-1}\mathbf{W}_{D,k})^{\rm T},\\
			& a_{F,ij}^{(r)} = (\tilde{\mathbf{h}}_{F,i}^{r-1}\mathbf{W}_{F,q})(\tilde{\mathbf{h}}_{D,j}^{r-1}\mathbf{W}_{F,k})^{\rm T},\\
			&e_{D,ij}^{r} = \frac{\exp(a_{D,ij}^{(r)})}{\sum_{k=1}^n \exp(a_{D,ik}^{(r)})},
			\\
			&e_{F,ij}^{r} = \frac{\exp(a_{F,ij}^{(r)})}{\sum_{k=1}^n \exp(a_{F,ik}^{(r)})},
		\end{aligned}
	\end{equation*}
	where $\tilde{\mathbf{h}}_{D,i}^{r-1}$ and $\tilde{\mathbf{h}}_{F,i}
	^{r-1}$ are the $i$-th rows of $\tilde {\displaystyle \mathbf{H}}_D^{(r-1)}$ and $\tilde {\displaystyle \mathbf{H}}_F^{(r-1)}$. Given the score matrices ${\displaystyle \mathbf{E}}_D^r=(e_{D,ij}^r)$ and ${\displaystyle \mathbf{E}}_F^r=(e_{F,ij}^r)$, the embeddings are updated via  aggregating values weighted by attention scores:
	\begin{equation}
		\label{eq:att}
		\begin{bmatrix}
			\tilde{\displaystyle \mathbf{H}}_D^{(r)}\\  \tilde{\displaystyle \mathbf{H}}_F^{(r)}
		\end{bmatrix} = \tilde {\displaystyle \mathbf{H}}^{(r-1)}+\text{MLP}^{r}\left[\begin{bmatrix}
			& {\displaystyle \mathbf{E}}_D^r\\
			{\displaystyle \mathbf{E}}_F^r & 
		\end{bmatrix}  
		\begin{bmatrix}
			\tilde{\displaystyle \mathbf{H}}_D^{(r-1)} \displaystyle \mathbf{W}_{D,k}^r\\
			\tilde{\displaystyle \mathbf{H}}_F^{(r-1)} \displaystyle \mathbf{W}_{F,k}^r
		\end{bmatrix}\right].
	\end{equation}
		Note that to avoid extending the node space to $n^2$, the feature is compressed into the $2n$ nodes and the matrices $\mathbf{D}$ and $\mathbf{F}$ are inserted into the block-diagonal pattern in~\eqref{eq:gcn}, leaving the unknown block off-diagonal relationship empty. Therefore, the attention scores $\mathbf{E}_D$ and $\mathbf{E}_F$ are calculated to reflect the dynamic relationship for the unknown pairs and fill in the block off-diagonal pattern in~\eqref{eq:att}. By leveraging both the block diagonal and off-diagonal relationships, the original two matrices $\mathbf{D}$ and $\mathbf{F}$ are embedded into two sets of node embeddings and their complex relationship is reflected.

	 Finally, the network generates a heatmap over facility-location pairs. Concretely, the heatmap is constructed via the dot product between facility and location  embeddings, followed by element-wise $\tanh$ clipping and a Sinkhorn normalization in the log-domain:
	\begin{equation*}
		\phi(\theta, \mathcal{P}) = \text{log-Sinkhorn}\left[C\tanh\left(\frac{\tilde{\displaystyle \mathbf{H}}_F^{(l_2)}(\tilde{\displaystyle \mathbf{H}}_D^{(l_2)})^{\rm T}}{\sqrt{d}}\right)\right],
	\end{equation*}
	where $C>0$ is a clipping constant. The $\tanh$ clipping here increases nonlinearity while preventing extreme logits. The log-Sinkhorn operator pushes the heatmap toward a doubly stochastic matrix in log-domain. This structure actually yields a low-rank approximation of some permutation matrix in the log-domain, as mentioned in~\cite{DrogeLBNH023lowrank}. Ideally we hope this approximated permutation matrix corresponds to an optimal solution to the problem in the log-domain.

	\subsection{Pre-Training with Pushforward Transformation}
	\label{subsec:local_search}
	
	We pretrain the probabilistic model by minimizing the pushforward loss~\eqref{eq:pushforward-obj},
	i.e., $\mathbb{E}_{P\sim\Gamma}\mathbb{E}_{\pi\sim p_\theta(\cdot\mid P)}[\,f(\mathcal{T}(\pi);P)\,]$,
	where $\mathcal{T}$ is the local improvement map.
	Conceptually, this training criterion treats the improved cost $f(\mathcal{T}(\pi);P)$ as the learning signal,
	thereby shaping the underlying EBM sampler to allocate probability mass to regions that consistently lead to
	high-quality solutions after a lightweight local refinement.
	
		Algorithm~\ref{alg:pretrain} summarizes the pretraining procedure.
		At each step we sample a mini-batch $\{P_i\}_{i=1}^B\sim \Gamma$ and, for each instance,
		draw $N$ permutations $\{\pi_{i,k}\}_{k=1}^N$ via parallel MH chains of length $L$.
	We use relatively long chains during pretraining to reduce sample correlation and stabilize learning in the cold-start regime.
	Each sample is refined once by $\hat\pi_{i,k}=\mathcal T(\pi_{i,k})$ and scored by $\hat f_{i,k}=f(\hat\pi_{i,k};P_i)$.
	We then update $\theta$ with the unbiased, variance-reduced gradient estimator \eqref{eq:gd-est}
		using the instance-wise baseline $b_i=\frac1N\sum_{k=1}^N \hat f_{i,k}$. This pretraining stage learns transferable structural priors jointly in the cross-graph attention network and the EBM,
	so that the model can produce strong zero-shot solutions and provide a favorable initialization for subsequent deployment-time adaptation.
	Empirically, we observe that the pretrained policy yields faster progress and better robustness on challenging benchmarks,
	as analyzed in Section~\ref{subsec:ablation}.
	
	\label{apx:subsec-pretrain}
	\begin{algorithm}[!htbp]
		\caption{Pre-training with Pushforward Transformation}
		\label{alg:pretrain}
		\begin{algorithmic}[1]
			\State \textbf{Input:} data distribution $\Gamma$, EBM $p_{\theta}$ with score function $\Phi_{\theta}$, number of training steps $S$, batch size B, number of samples $N$ per instance, length of chains $L$. 
			\For{$s=1,...,S$}
			\State Sample a batch of instances $\{\mathcal{P}_i\}_{i\in[B]}$ from $\Gamma$.
			\State Use the MH sampler with $N$ chains of length $L$ to sample $\{\pi_{i,k}\}_{k\in[N]}$ from $p_{\theta}(\cdot\mid\mathcal{P}_i)$ for each instance $\mathcal{P}_i$, where $i=1,\dots, B$.
			\State Apply local improvement map to obtain $\hat{\pi}_{i,k} = \mathcal{T}(\pi_{i,k})$. \hfill $(i\in[B],\,k\in[N])$
			\State Evaluate objective for improved samples: $\hat{f}_{i,k}= f(\hat{\pi}_{i,k};\mathcal{P}_i)$. \hfill $(i\in[B],\,k\in[N])$
			\State $b_i\leftarrow \frac{1}{N}\sum\limits_{k=1}^N\hat{f}_{i,k},\,i\in[B].$
			\State $\hat{G}\leftarrow \frac{1}{B(N-1)}\sum\limits_{i=1}^B\sum\limits_{k=1}^N(\hat{f}_{i,k} - b_i)\nabla_\theta \Phi_{\theta}(\pi_{i,k}).$
			\State $\theta \leftarrow \mathrm{Adam}(\theta, \hat{G})$.
			\EndFor
			
		\end{algorithmic}
	\end{algorithm}

	\subsection{Batch-Wise Warm-Started MCMC Finetuning}
	To enhance generalization, we introduce an MCMC-based finetuning stage that efficiently adapts the model to new target instances.  As detailed in Algorithm \ref{alg:finetune}, the procedure begins by generating $K$ initial states $\{\pi_{i,k}\}_{k\in[K]}$  for each problem instance $\mathcal{P}_i$  using a long-run MH sampler to ensure high-quality starting points. At each subsequent finetuning step, $M$ parallel Markov chains of length $L$ are launched from every $\pi_{i,k}$, producing  terminal states $\{\pi_{i,k}^m\}_{m\in[M]}$. These terminal states are then refined by a local improvement map to $\hat{\pi}_{i,k}^m=\mathcal{T}(\pi_{i,k}^m)$, after which the objectives of these refined samples are evaluated. This entire process of sampling, improvement and evaluation runs in parallel for all instances in a batch. Next, the resulting samples are aggregated to estimate a single policy gradient which is used to perform a synchronized update on the shared model parameters. Finally, the initial states for the next iteration follow a within group best retention rule. For each $k$, the element of $\{\hat{\pi}_{i,k}^m\}_{m\in[M]}$ with the smallest objective value is selected as the next initial state. In this way, high quality solutions from the previous iteration are preserved as warm starts that steer subsequent sampling toward promising regions and promote steady convergence. 
	
	Our finetuning is performed in a batch-wise manner, enabling the network to simultaneously explore for high-quality solutions across different instances. In contrast to traditional per-instance search methods that adjust model parameters separately for each test instance, this highly parallel approach enhances computational efficiency and leads to a considerable reduction in runtime.  While instances within a batch share a common set of network parameters during the update, our experiment in Appendix~\ref{apx:subsec-batch-wise} confirms that this shared adaptation scheme does not impede final solution quality.

	The efficacy of our finetuning framework hinges on two core design principles. The first is the dynamic use of Markov chain lengths. For generating the initial set of start points, we employ long chains, which ensures the initial points are high-quality samples drawn faithfully from the model to leverage the power of pretraining. Once finetuning commences, we switch to using short chains which facilitate the preservation of structural information carried by the initial states. The second is the mechanism for interplay across chains, which is crucial for striking a  trade off between solution quality and sample diversity.  If all filtered samples are indiscriminately pooled as starting points for the next iteration, diversity is abundant but instability arises from low quality chains. If only globally best samples are retained, trajectories derived from the same initial state tend to dominate and diversity deteriorates. In contrast, our within group competitive mechanism suppresses the adverse influence of inferior chains while preventing the loss of global diversity.

	\begin{algorithm}[!htbp]
		\caption{Batch-wise warm-started MCMC finetuning}
		\label{alg:finetune}
		\begin{algorithmic}[1]
			\State \textbf{Input:} pretrained EBM $p_{\theta}$ with score function $\Phi_{\theta}$, a batch of instances $\{\mathcal{P}_1,\dots,\mathcal{P}_B\}$, number of starting points $K$, number of chains per sample $M$,  length of chains $L$, number of epochs $T$. 
			\State  Use a long-run MH sampler to generate initial points $\pi_{i,k}\sim p_{\theta}(\cdot\mid\mathcal{P}_i)$. \hfill $(i\in[B], k\in[K])$
			\For{$t=1,...,T$} 
			\For{each instance $i \in [B]$ \textbf{in parallel}}
			\State \parbox[t]{0.8\linewidth}{
				Starting from $\pi_{i,k}$, use the MH sampling with $M$ chains of length $L$
				to sample $\{\pi_{i,k}^m\}_{m\in[M]}$ from distribution
				$p_{\theta}(\cdot\mid\mathcal{P}_i)$. \hfill $(k\in[K])$
			}
			
			\State \parbox[t]{0.8\linewidth}{
				Apply local improvement map to obtain
				$\hat{\pi}_{i,k}^m = \mathcal{T}(\pi_{i,k}^m)$. \hfill $(k\in[K],\,m\in[M])$
			}
			
			\State \parbox[t]{0.8\linewidth}{
				Evaluate objective for improved samples:
				$\hat{f}_{i,k}^m = f(\hat{\pi}_{i,k}^m;\mathcal{P}_i)$. \hfill $(k\in[K],\,m\in[M])$
			}
			\EndFor
			\State $b_i\leftarrow \frac{1}{KM}\sum\limits_{k=1}^K\sum\limits_{m=1}^M\hat{f}^m_{i,k},\,i\in[B].$
			\State $\hat{G}\leftarrow \frac{1}{B(KM-1)}\sum\limits_{i=1}^B\sum\limits_{k=1}^K\sum\limits_{m=1}^M (\hat{f}_{i,k}^m - b_i)\nabla_\theta \Phi_{\theta}(\pi_{i,k}^m)$.
			\State $\theta \leftarrow \mathrm{Adam}(\theta, \hat{G})$.
			\State  $\pi_{i,k} \leftarrow \operatorname*{arg\,min}\bigl\{\, f(\hat{\pi}) \mid \hat{\pi}\in\{\hat{\pi}_{i,k}^m\}_{m\in[M]} \,\bigr\}. \hfill (i\in[B],\,k\in[K])$
			\EndFor
		\end{algorithmic}
	\end{algorithm}
	
	\section{Application in the BM Problem}\label{sec:application-BM}
	In this section, we demonstrate how the proposed QAP solver PLMA can be applied to the BM problem. The BM problem is NP-hard and arises in a variety of applications, including sparse matrix computations and circuit design. Given an undirected graph $G=(V,E)$ with vertex set $V=\{1,\dots,n\}$ and edge set $E$, the BM problem aims to find a permutation $\pi\in\Pi_n$ that minimizes the graph bandwidth, i.e.,
	\[
	\min_{\pi\in\Pi_n}\ \max_{(i,j)\in E} |\pi(i)-\pi(j)|.
	\]
	The BM problem can also be formulated in matrix form. Let $A\in\mathbb{R}^{n\times n}$ be the adjacency matrix of $G$. The bandwidth of $A$ is defined by $b(A)=\max_{A_{ij}\neq 0}|i-j|$.
	Then the BM problem is equivalent to finding a permutation matrix $X$ such that the bandwidth of the permuted matrix $XAX^{\rm T}$ is minimized, namely,
	\[
	\min_{X\in\Pi_n} b(XAX^{\rm T}).
	\]
	
	A standard way to solve the BM problem is to reduce it to a sequence of QAPs. For each integer $m\in\{0,1,\dots,n-1\}$, define the symmetric Toeplitz matrix $B_m\in\mathbb{R}^{n\times n}$ by $(B_m)_{ij}=\max\{|i-j|-m,\,0\},$
	and consider
	\begin{equation}
		\label{eq:hm}
		h(m)=\min_{X\in\Pi_n}\langle B_m, XAX^{\rm T}\rangle.
	\end{equation}
	By construction, $h(m)=0$ if and only if there exists a permutation matrix $X$ such that $b(XAX^{\rm T})\le m.$
	Therefore, the BM problem can be reformulated as the problem of finding the smallest nonnegative integer $m$ such that $h(m)=0$, namely,
	\begin{equation}
		\label{eq:bm}
		\min \quad m \qquad \text{s.t.}\quad h(m)=0.
	\end{equation}
	
	Let $m^*$ denote the optimal bandwidth. Then $m^*$ is the unique solution of \eqref{eq:bm}. Moreover, $h(m)=0$ for all $m\ge m^*$, and $h(m)>0$ for all $0\le m<m^*$. Hence, $h(m)$ induces a monotone feasibility criterion, which allows us to employ a bisection strategy. At each iteration, we solve the QAP in \eqref{eq:hm} approximately using PLMA and then update the search interval according to whether the obtained objective value is zero. In this way, we obtain a bisection-based PLMA method, termed Bi-PLMA, for solving the BM problem. The procedure is summarized in Algorithm~\ref{alg:bi-plma}.
	
	\begin{algorithm}[!htbp]
		\caption{Bi-PLMA for the BM problem}
		\label{alg:bi-plma}
		\begin{algorithmic}[1]
			\State \textbf{Input:} Adjacency matrix $A$, initial permutations $\pi_{\mathrm{init}}$, probabilistic model $p_\theta$, and hyperparameters.
			\State Initialize the lower and upper bounds by $\underline{m}=0$ and $\overline{m}=n-1$.
			\While{$\overline{m}-\underline{m}>1$}
			\State Set $m=\left\lceil(\underline{m}+\overline{m})/2\right\rceil$.
			\State Solve the QAP subproblem
			\[
			\hat{h}(m)\leftarrow \text{PLMA}(A,B_m,p_\theta,\pi_{\mathrm{init}})
			\]
			and update the model $p_\theta$ together with the initial permutations $\pi_{\mathrm{init}}$.
			\If{$\hat{h}(m)>0$}
			\State $\underline{m}\leftarrow m$.
			\Else
			\State $\overline{m}\leftarrow m$.
			\EndIf
			\EndWhile
			\State \textbf{Return:} $\overline{m}$ and the associated permutation $\pi$.
		\end{algorithmic}
	\end{algorithm}
	
	In the BM setting, the bisection procedure requires solving a sequence of QAP subproblems with matrix pairs $(A,B_m)$, where the matrix $B_m$ varies with the current bandwidth threshold $m$. To better exploit the shared structure among these subproblems, we replace the original instance-conditioned network with a simple parameterization that treats $\theta\in\mathbb{R}^{n\times n}$ as the learnable variable, and define the heatmap by
	\[
	\phi=\phi(\theta)=\text{log-Sinkhorn}(C\theta),
	\]
	where $C>0$ is a clipping constant. This design makes the heatmap independent of the particular choice of $B_m$. Although the QAP objective changes as $m$ varies, all subproblems are induced by the same underlying BM task and therefore share the same implicit optimization goal, namely, searching for a permutation with minimum bandwidth. Consequently, during the bisection process, the learned heatmap can be continuously updated and passed from one QAP subproblem to the next, yielding a natural warm-start effect.
		\begin{table*}[!t]
		\caption{Performance comparison on synthetic datasets (256 instances). Ro-TS (xk) denotes Ro-TS with x*1000*n iterations. IPFP (x) runs IPFP from x random initializations. PLMA ($T=1$) represents the zero-shot performance of the pre-trained model without any fine-tuning, while $T=50/200$ correspond to the results after 50 / 200 fine-tuning iterations.}
		\label{tab:synthetic}
		\centering
		\footnotesize
		\setlength{\tabcolsep}{4pt}
		\renewcommand{\arraystretch}{0.9}
		\begin{tabular}{l rrr @{\hspace{1.5em}} rrr @{\hspace{1.5em}} rrr}
			\toprule
			& \multicolumn{3}{c}{\textbf{QAP20}} & \multicolumn{3}{c}{\textbf{QAP50}} & \multicolumn{3}{c}{\textbf{QAP100}} \\
			\cmidrule(lr){2-4} \cmidrule(lr){5-7} \cmidrule(lr){8-10}
			Algorithm & Cost & Gap & Time & Cost & Gap & Time & Cost & Gap & Time \\
			\midrule
			\multicolumn{10}{c}{\textbf{\textit{Geometrically Structured Datasets}}} \\
			\midrule
			Ro-TS (1k)       & 54.38 & 0.01\% & 17.04s & 375.99 & 0.14\% & 4m35s & 1593.27 & 0.13\% & 38m56s \\
			Ro-TS (5k)       & \textbf{54.37} & \textbf{0.00\%} & 1m25s & \textbf{375.48} & \textbf{0.00\%} & 22m53s & 1591.25 & 0.00\% & 3h15m \\
			BMA       & \textbf{54.37} & \textbf{0.00\%} & 1m57s & 375.60 & 0.03\% & 15m53s & 1591.58 & 0.02\% & 2h21m \\
            C-SA & 54.53 & 0.28\% & 1m57s & 376.56 & 0.29\% & 21m32s & 1592.95& 0.11\% & 2h45m \\
			\cmidrule(lr){1-10}
			IPFP            & 55.11 & 1.37\% & 2.04s & 378.76 & 0.88\% & 11.47s & 1600.27 & 0.57\% & 1m34s \\
			
			IPFP (10)        & 54.54 & 0.31\% & 20.97s & 376.60 & 0.30\% & 2m30s & 1594.76 & 0.22\% & 17m34s \\
			RRWM            & 71.30 & 31.09\% & 11.30s & 428.78 & 14.14\% & 39.23s & 1700.33 & 6.86\% & 6m32s \\
			SM              & 64.38 & 18.45\% & 0.22s & 426.92 & 13.70\% & 7.14s & 1753.10 & 10.17\% & 1m40s \\
			\cmidrule(lr){1-10}
			NGM             & 62.93 & 15.87\% & 24.78s & 429.69 & 14.46\% & 1m16s & 1773.71 & 11.47\% & 2m29s \\
			SAWT (10k)       & 54.72 & 0.64\% & 3m41s & 380.92 & 1.45\% & 5m36s & 1617.30 & 1.64\% & 10m43s \\
			\cmidrule(lr){1-10}
			
			\textbf{PLMA ($T=1$)} & 54.63 & 0.48\% & 0.06s &379.79 & 1.15\% & 0.41s&1607.84 & 1.04\%  & 4.27s\\
			\textbf{PLMA ($T=50$)} & \textbf{54.37} & \textbf{0.00\%} & 2.57s &375.55 & 0.20\%& 19.88s &1591.73 & 0.03\%  & 3m30s\\
			\textbf{PLMA ($T=200$)}  &\textbf{54.37} &\textbf{0.00\%}& 9.36s& \textbf{375.48} & \textbf{0.00\%} & 1m19s & \textbf{1591.23} & \textbf{0.00\%} & 13m58s \\
			
			\midrule
			\multicolumn{10}{c}{\textbf{\textit{Uniformly Random Datasets}}} \\
			\midrule
			Ro-TS(1k)       & 76.61 & 0.07\% & 17.56s & 523.08 & 0.22\% & 4m36s & 2195.98 & 0.13\% & 38m59s \\
			Ro-TS(5k)       & \textbf{76.56} & \textbf{0.00\%} & 1m27s & 521.91 & 0.00\% & 22m59s & 2193.16 & 0.00\% & 3h15m \\
			BMA       & \textbf{76.56} & \textbf{0.00\%} & 1m36s & 521.79 & -0.02\% & 17m26s & 2194.14 & 0.04\% & 3h32m \\
            
              C-SA & 76.64 & 0.11\% & 2m2s & 523.19 & 0.25\% & 22m21s & 2193.69 & 0.02\% & 2h45m \\
			\cmidrule(lr){1-10}
			IPFP            & 79.13 & 3.39\% & 2.12s & 530.74 & 1.69\% & 7.45s & 2211.38 & 0.83\% & 41.95s \\
			IPFP(25)        & 77.60 & 1.37\% & 54.70s & 526.96 & 0.97\% & 4m20s & 2203.29 & 0.46\% & 19m20s \\
			RRWM            & 93.50 & 22.27\% & 11.11s & 592.50 & 13.54\% & 31.91s & 2432.34 & 10.91\% & 5m6s \\
			SM              & 92.32 & 20.73\% & 0.19s & 605.08 & 15.95\% & 5.65s & 2457.47 & 12.05\% & 1m23s \\
			\cmidrule(lr){1-10}
			NGM             & 88.34 & 15.49\% & 25.08s & 594.99 & 14.01\% & 1m17s & 2438.52 & 11.19\% & 2m29s \\
			\cmidrule(lr){1-10}
			\textbf{PLMA ($T=1$)}        &  78.23 & 2.20\% & 0.06s & 538.42 & 3.17\% & 0.40s  & 2243.43 & 2.29\%&  4.32s   \\
			\textbf{PLMA ($T=50$)}        &  76.59 & 0.05\% & 2.47s & 523.95 & 0.40\% & 19.51s  & 2200.40 & 0.34\%&  3m30s   \\
			\textbf{PLMA ($T=200$)}        & \textbf{76.56} & \textbf{0.00\%} & 9.60s & \textbf{521.75} &  \textbf{-0.03\%} & 1m18s & \textbf{2193.13} & \textbf{0.00\%} & 14m1s \\
			\bottomrule
		\end{tabular}
	\end{table*}
	
	\section{Experiments}\label{sec:experiments}
	\subsection{Experimental Setup}
	All experiments are conducted on a server with NVIDIA Tesla A800 GPUs (80GB) and Intel Xeon Gold 6326 CPUs (256GB) at  2.90GHz.
	
	\textbf{Datasets.} We evaluate the in-distribution performance of our model on two families of synthetic QAP instances using sizes $n\in\{20,50,100\}$. The synthetic data comprise geometrically structured instances generated as in \cite{tanlearning} and uniformly random instances produced following \cite{Taillard1991}. Out-of-distribution performance is assessed on large-scale instances with $n\in\{200,500\}$, where we generate 64 uniformly random instances for each large size, and on the real-world QAPLIB \cite{burkard1997qaplib} and Taixxeyy \cite{drezner2005taie} benchmarks. For out-of-distribution evaluation, the pretraining is performed with uniformly random instances with $n=100$. To clarify the structural differences between the training and testing distributions, we describe the generation mechanisms for each dataset in Appendix~\ref{subsec:datasets}.

	\textbf{Baselines.} We compare our model against a wide spectrum of established and modern approaches, grouped into three categories.
	(i) For search-based solvers, we consider strong representative baselines, including robust tabu search (Ro-TS) \cite{Taillard1991}, a highly optimized heuristic solver, BMA~\cite{BMA}, a memetic search algorithm that combines the genetic algorithm and breakout local search, and Connolly's simulated annealing algorithm (C-SA)~\cite{connolly1990SA}.
	(ii) For heuristic solvers, we benchmark against classic iterative algorithms often used in graph matching, including IPFP~\cite{Leordeanu2005ipfp}, SM~\cite{LeordeanuH05sm}, and RRWM~\cite{ChoLL10rrwm}.
	(iii) For learning-based solvers, we include two recent deep learning approaches, SAWT~\cite{tanlearning} and NGM~\cite{wang2021neural}. Implementation details for these baselines are available in Appendix~\ref{apx:subsec-implementation-baseline}. 
	
	\textbf{Metrics.} We evaluate all methods using three metrics: (i) Cost, the average objective value over all test instances; (ii) Gap, the average relative gap from a reference solution $C_{\text{ref}}$, computed as $(C_{\text{alg}} - C_{\text{ref}}) / C_{\text{ref}} \times 100\%$, where $C_{\text{ref}}$ represents the best known solution for benchmarks (QAPLIB, Taixxeyy) or the Ro-TS solution for synthetic datasets; and (iii) Time, the total computation time to solve the entire test set.
	
	\subsection{Main Benchmark Results}
    \subsubsection{Results on synthetic QAP datasets}
	The results in Table~\ref{tab:synthetic} establish PLMA as the new state-of-the-art. With 200 finetuning steps, PLMA consistently matches or surpasses the strongest handcrafted baselines, including Ro-TS, BMA, and C-SA, while remaining remarkably efficient; on challenging QAP100 instances, it achieves an optimality gap of 0.00 \% over 13 times faster than Ro-TS (5k). The results also validate our two-stage design. The pre-trained model ($T=1$) instantly yields high-quality results that surpass other learning-based methods (NGM, SAWT). With brief fine-tuning ($T=50$), the model rapidly converges to near-optimality, demonstrating the synergy between a strong pre-trained starting distribution and efficient MCMC-based refinement.
	
	\begin{table}[t]
		\caption{Performance comparison on large-scale synthetic datasets (64 instances).}
		\label{tab:large-scale}
		\centering
		\footnotesize
		\setlength{\tabcolsep}{1.5pt}
		\begin{tabular}{l rrr @{\hspace{1.5em}} rrr}
			\toprule
			& \multicolumn{3}{c}{\textbf{QAP200}} & \multicolumn{3}{c}{\textbf{QAP500}} \\
			\cmidrule(lr){2-4} \cmidrule(lr){5-7} 
			Algorithm & Cost & Gap & Time & Cost & Gap & Time  \\
			\midrule
			Ro-TS       & 6640.28 & 0.00\% & 3h48m & 43028.56 & 0.00\% & 29h59m \\
			BMA       & 6641.26 & 0.01\% & 2h58m & 42989.26 & -0.09\% & 23h16m \\
			
			SAWT (10k)   & 6735.97 & 1.44\% & 8m33s & 43696.84 & 1.55\% & 48m36s \\

            	\textbf{PLMA ($T=1$)} & 6675.55 & 0.53\% & 7.50s & 43190.13 & 0.38\% & 1m31s \\
                			\textbf{PLMA ($T=50$)}  & \textbf{6639.93} & \textbf{-0.01\%} & 6m15s & \textbf{42975.27} & \textbf{-0.12\%}  & 1h15m \\
			
			\bottomrule
		\end{tabular}
	\end{table}

	\begin{table*}[!t]
		\caption{Performance on QAPLIB. Average gaps and computation time (s) are reported for each category. The search terminates upon finding the optimal solution or reaching the iteration limit.}
		\label{tab:qaplib_results}
		\centering
		\setlength{\tabcolsep}{2.5pt}
		\begin{tabular}{l cccccccccc}
			\toprule
			& \multicolumn{2}{c}{Ro-TS} & \multicolumn{2}{c}{BMA} & \multicolumn{2}{c}{SAWT} & \multicolumn{2}{c}{IPFP} & \multicolumn{2}{c}{PLMA} \\
			\cmidrule(lr){2-3} \cmidrule(lr){4-5} \cmidrule(lr){6-7} \cmidrule(lr){8-9} \cmidrule(lr){10-11}
			Class & Gap & ${\text{Time}}$ & Gap & ${\text{Time}}$ & Gap & ${\text{Time}}$ & Gap & ${\text{Time}}$ & Gap & ${\text{Time}}$ \\
			\midrule
			bur (26) & \textbf{0.00\%} & 0.12 & \textbf{0.00\%} & 0.03 & 3.95\% & 14.67 & 0.05\% & 0.33 & \textbf{0.00\%} & 0.08 \\
			chr (12-25) & 0.48\% & 0.16 & 0.18\% & 0.14 & 147.54\% & 14.18 & 14.90\% & 0.26 & \textbf{0.00\%} & 0.31 \\
			els (19) & \textbf{0.00\%} & 0.02 & \textbf{0.00\%} & 0.02 & 47.37\% & 14.24 & 10.18\% & 0.36 & \textbf{0.00\%} & 0.09 \\
			esc (16-128) & \textbf{0.00\%} & 0.62 & \textbf{0.00\%} & 0.01 & 43.29\% & 15.07 & 0.39\% & 0.63 & \textbf{0.00\%} & 0.07 \\
			had (12-20) & \textbf{0.00\%} & 0.00 & \textbf{0.00\%} & 0.00 & 5.17\% & 14.23 & 0.08\% & 0.39 & \textbf{0.00\%} & 0.05 \\
			kra (30-32) & \textbf{0.00\%} & 0.24 & \textbf{0.00\%} & 0.07 & 32.92\% & 14.77 & 0.65\% & 0.62 & \textbf{0.00\%} & 0.31 \\
			lipa (20-90) & 0.03\% & 3.28 & \textbf{0.02\%} & 2.35 & 1.40\% & 17.32 & 1.07\% & 4.93 & 0.08\% & 1.57 \\
			nug (12-30) & \textbf{0.00\%} & 0.02 & \textbf{0.00\%} & 0.02 & 19.25\% & 14.39 & 0.05\% & 0.38 & \textbf{0.00\%} & 0.17 \\
			rou (12-20) & \textbf{0.00\%} & 0.04 & \textbf{0.00\%} & 0.04 & 15.09\% & 14.25 & 0.77\% & 0.33 & \textbf{0.00\%} & 0.08 \\
			scr (12-20) & \textbf{0.00\%} & 0.01 & \textbf{0.00\%} & 0.01 & 33.92\% & 14.22 & 1.24\% & 0.28 & \textbf{0.00\%} & 0.08 \\
			sko (42-100) & 0.05\% & 29.22 & \textbf{0.03\%} & 21.00 & 16.17\% & 19.06 & 0.30\% & 8.40 & \textbf{0.03\%} & 7.01 \\
			ste (36) & 0.01\% & 0.53 & \textbf{0.00\%} & 0.30 & 107.95\% & 14.95 & 1.81\% & 0.70 & \textbf{0.00\%} & 0.55 \\
			tai (10-256) & 0.23\% & 25.16 & 0.21\% & 19.06 & 34.67\% & 16.62 & 0.92\% & 3.33 & \textbf{0.20\%} & 4.73 \\
			tho (30-150) & \textbf{0.04\%} & 38.74 & \textbf{0.04\%} & 26.59 & 24.05\% & 17.19 & 0.42\% & 12.46 & \textbf{0.04\%} & 6.42 \\
			wil (50-100) & \textbf{0.02\%} & 26.06 & \textbf{0.02\%} & 18.41 & 9.52\% & 17.64 & 0.12\% & 9.66 & \textbf{0.02\%} & 6.82 \\
			\midrule
			\textbf{Average} & 0.11\% & 9.68 & 0.07\% & 7.06 & 37.82\% & 15.85 & 2.15\% & 2.73 & \textbf{0.06\%} & 2.16 \\
			\bottomrule
		\end{tabular}
	\end{table*}

    	\begin{table*}[t]
		\caption{Grouped results on Taixxeyy instances, with all metrics averaged over 10 independent runs. For each class, the reported mean and [min, max] gaps (\%) are averages of per-instance statistics.}
		\label{tab:tai_results}
		\centering
		
		\setlength{\tabcolsep}{3pt} 
		\begin{tabular}{l ccc ccc ccc ccc}
			\toprule
			& \multicolumn{3}{c}{Ro-TS} & \multicolumn{3}{c}{BMA} & \multicolumn{3}{c}{IPFP} & \multicolumn{3}{c}{PLMA} \\
			\cmidrule(lr){2-4} \cmidrule(lr){5-7} \cmidrule(lr){8-10} \cmidrule(lr){11-13}
			Class & mean & [min, max] & Time (s) & mean & [min, max] & Time (s) & mean & [min, max] & Time (s) & mean & [min, max] & Time (s) \\
			\midrule
			tai27e & 41.50 & [0.11, 221.08] & 0.57 & 0.12 & [0.00, 0.67] & 0.18 & 19.47 & [6.14, 34.11] & 0.38 & \textbf{0.00} & [0.00, 0.00] & 0.08 \\
			tai45e & 101.89 & [1.00, 400.60] & 3.83 & 8.39 & [0.74, 17.63] & 2.21 & 22.01 & [8.26, 36.98] & 1.03 & \textbf{0.00} & [0.00, 0.00] & 0.18 \\
			tai75e & 111.28 & [6.20, 280.01] & 18.49 & 15.41 & [5.77, 22.76] & 11.21 & 27.64 & [17.56, 36.78] & 2.52 & \textbf{0.08} & [0.00, 0.63] & 1.27 \\
			tai125e & 82.53 & [7.65, 265.54] & 72.52 & 16.51 & [10.56, 21.72] & 66.01 & 26.93 & [20.64, 32.67] & 8.52 & \textbf{3.65} & [0.78, 6.32] & 8.65 \\
			tai175e & 67.86 & [9.11, 260.98] & 158.18 & 21.07 & [15.15, 26.56] & 193.65 & 23.32 & [17.70, 28.11] & 16.86 & \textbf{9.09} & [5.96, 12.24] & 14.43 \\
			\midrule
			Average & 81.01 & [4.82, 285.64] & 50.72 & 12.30 & [6.44, 17.87] & 54.65 & 23.87 & [14.06, 33.73] & 5.86 & \textbf{2.56} & [1.35, 3.84] & 4.92 \\
			\bottomrule
		\end{tabular}
	\end{table*}
	
	\begin{table*}[!t]
		\centering
		\includegraphics[width=\textwidth]{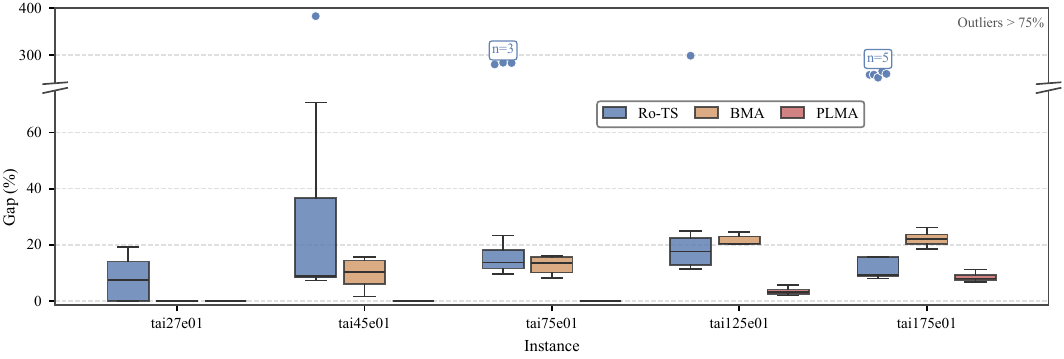}
		\captionof{figure}{Gap distributions on representative Taixxeyy instances over 10 independent runs, comparing PLMA against Ro-TS and BMA on \texttt{tai27e01}, \texttt{tai45e01}, \texttt{tai75e01}, \texttt{tai125e01}, and \texttt{tai175e01}. The main lower axis focuses on the boxplot body, while large-gap outliers are shown separately in the upper broken axis.}
		\label{fig:tai_gap_distribution}
	\end{table*}
	
	\subsubsection{Results on large-scale synthetic datasets}
	To assess the scalability of PLMA, we evaluate its performance on instances of a significantly larger scale ($n\in\{200,500\}$). The results in Table~\ref{tab:large-scale} confirm that our approach scales effectively to these larger problem sizes. Remarkably, PLMA remains competitive with or even superior to strong handcrafted baselines such as Ro-TS and BMA, while requiring only a fraction of the computational cost.

	\subsubsection{Results on QAPLIB}
	We evaluate PLMA on the widely-used QAPLIB benchmark, with results summarized in Table~\ref{tab:qaplib_results}. PLMA achieves the best overall performance, attaining the lowest average optimality gap of 0.06\% while also requiring the least average computation time among all heuristic baselines. These results demonstrate the strong generalization ability and practical efficiency of PLMA on diverse real-world instances.
	
	\subsubsection{Results on Taixxeyy instances} We test robustness on challenging Taixxeyy instances over 10 independent runs. In Table~\ref{tab:tai_results}, we report grouped results. PLMA delivers a low average gap of 2.56\% and a reliable worst-case performance with an average maximum gap of 3.84\%. In contrast, handcrafted local-search baselines are markedly less stable: Ro-TS shows a highly erratic 81.01\% average gap with catastrophic failures up to 285.64\%, while BMA also remains clearly inferior in both average and worst-case performance. To further assess robustness, we visualize the distribution of final gaps in Figure~\ref{fig:tai_gap_distribution}. The main lower axis focuses on the central boxplot region, while large-gap failures are shown separately as individual points in the upper broken axis. In contrast to Ro-TS and BMA, which exhibit noticeable variability and extreme failures, PLMA yields no outliers and its boxes are distinctly narrower, indicating stable behavior across trials.
	\begin{figure*}[!t]
		\centering
		\includegraphics[width=\textwidth]{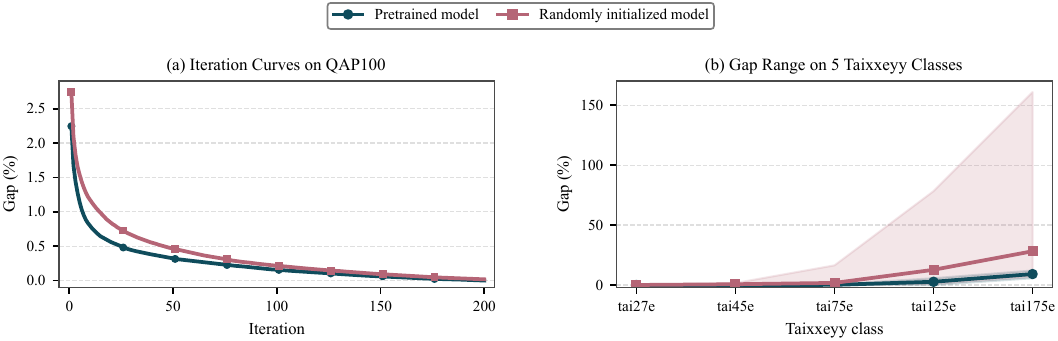}
		\caption{Performance comparison between finetuning from a pretrained model and finetuning from a randomly initialized model. Panel (a) reports the iteration curves on uniformly random dataset ($n=100$). Panel (b) reports the gap range over five Taixxeyy classes, where the solid lines denote the mean gap and the shaded regions indicate the interval between the minimum and maximum gaps over 10 independent trials.}
		\label{fig:pretrain_vs_random}
	\end{figure*}
	\begin{figure*}[!t]
		\centering
		\includegraphics[width=\textwidth]{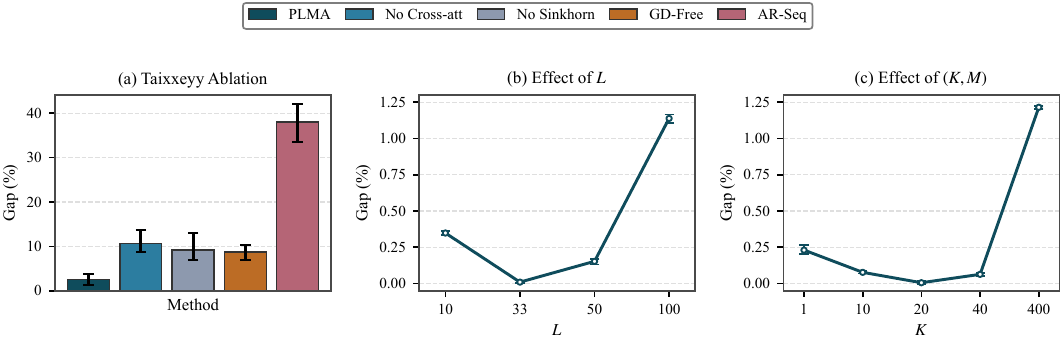}
		\caption{Ablation and hyperparameter study. Panel (a) summarizes the Taixxeyy ablation results using the mean gap and the $[\min,\max]$ range aggregated over the five Taixxeyy classes. Panels (b) and (c) report the effect of the Markov chain length $L$ and the sampling configuration $(K,M)$ on uniformly random dataset with $n=100$, respectively, where the error bars indicate the minimum and maximum gaps over 10 independent runs. In panel (c), the horizontal axis uses $K$ only, while the total sample budget is fixed by keeping $K \times M = 400$.}
		\label{fig:ablation_hyperparameter}
	\end{figure*}

	\subsection{Ablation Study}
	\label{subsec:ablation}
    We conduct ablation studies on pretraining, network design, and the finetuning strategy. To evaluate pretraining, we compare finetuning from a pretrained model against finetuning from a randomly initialized model on both the uniformly random dataset and the Taixxeyy benchmark. The former provides the in-distribution evaluation, whereas the latter serves as the out-of-distribution evaluation. 
    
    As reported in Figure~\ref{fig:pretrain_vs_random}(a), the pretrained model reaches lower optimality gaps in fewer iterations on the in-distribution dataset. This confirms that the transferable knowledge from pretraining provides a critical advantage in finetuning efficiency, enabling the model to provide superior solutions faster in the early stages of finetuning. At the same time, with sufficiently many finetuning iterations, the randomly initialized model can eventually approach the performance of the pretrained one. This behavior further supports the effectiveness of our warm-started MCMC finetuning. The results in Figure~\ref{fig:pretrain_vs_random}(b)  show that the pretrained model attains lower mean gaps and a narrower performance range across all five Taixxeyy classes. This  indicates that pretraining equips the model with general structural knowledge of the QAP that is necessary for adapting to these difficult instances. Without such priors, the randomly initialized model struggles to handle the complexity of Taixxeyy instances and fails to reach high-quality solutions even after many finetuning iterations.

    The network ablation on Taixxeyy is summarized in Figure~\ref{fig:ablation_hyperparameter}(a). Removing either the cross-attention module  (No Cross-att) or the Sinkhorn normalization (No Sinkhorn) degrades solution quality. For a fair comparison, No Cross-att replaces the removed attention block with additional GNN layers, so that its parameter count and inference time remain close to those of the full model.  These results demonstrate the essential roles these components play in fusing inter-graph information and structuring the probabilistic search space.

    The results in Figure~\ref{fig:ablation_hyperparameter}(a) also evaluate the finetuning strategy, where PLMA is compared with AR-Seq and GD-Free. AR-Seq replaces our energy-based formulation with an autoregressive factorization induced by the learned heatmap and constructs a permutation sequentially through conditional sampling. GD-Free performs MCMC sampling with 2-swap local improvement, but without gradient updates. The results reveal two key advantages of our approach. First, it overcomes the restart-from-scratch limitation of AR models by enabling warm-starts from prior solutions. Second, the superiority over GD-Free underscores the critical role of integrating 2-swap operations into the learning process, rather than merely applying them as post-processing.

	\subsection{Hyperparameter Study}
	To analyze the sensitivity of our framework and validate its design, we conduct a hyperparameter study for the warm-started MCMC finetuning stage (Algorithm~\ref{alg:finetune}). This analysis is performed on the uniformly random dataset ($n=100$), and the results are shown in Figure~\ref{fig:ablation_hyperparameter}(b)-(c). We focused on two key hyperparameter choices governing the sampling process: the Markov chain length $L$, and the allocation of a fixed sample budget across the number of initial starting points $K$ and the number of parallel chains per starting point $M$.
	
	Figure~\ref{fig:ablation_hyperparameter}(b) shows the impact of Markov chain length  $L$. The results confirm that a short-chain configuration is highly effective for adaptation. Optimal performance is achieved with $L = \lfloor n/3 \rfloor \approx 33$. This finding supports our core design principle: since the chains are warm-started from high-quality solutions found in the previous iteration, only a limited local exploration is needed for refinement. Excessively long chains risk straying from the promising region and wasting computation, while chains that are too short may not adequately explore the local neighborhood.
	
	   In Figure~\ref{fig:ablation_hyperparameter}(c), we study the balance between the diversity of starting points (controlled by $K$) and the search intensity around each point (controlled by $M$), keeping the total number of samples constant ($K \times M = 400$). The results show that extreme configurations perform poorly. For instance, using many starting points with minimal exploration each ($(K, M) = (400, 1)$) lacks the search intensity to refine solutions, whereas concentrating all resources on a single starting point ($(K, M) = (1, 400)$) suffers from a lack of initial diversity. The balanced configuration of $(K, M) = (20, 20)$ yields the best results, highlighting the benefit of exploring from a diverse set of promising solutions, with a sufficient search effort allocated to each.

	\subsection{Results on the Bandwidth Minimization Problem}
	We compare the proposed Bi-PLMA with three representative baseline methods. Specifically, we consider the classical reverse Cuthill--McKee (rCM) algorithm~\cite{rcm}, the improved reverse Cuthill--McKee (irCM) algorithm~\cite{ircm}, and a bisection-based variant built upon Ro-TS, referred to as Bi-RoTS, which is obtained by replacing the QAP solver in Algorithm~\ref{alg:bi-plma} with Ro-TS. The test set consists of large-scale sparse matrices with $120\le n\le 300$ selected from the UF Sparse Matrix Collection~\cite{sparsematrix}. Since rCM is extremely fast and usually provides a good feasible solution, we use the bandwidth returned by rCM as the initial upper bound $\overline m$ in Algorithm~\ref{alg:bi-plma}, which further accelerates the bisection procedure.
	
	In Table~\ref{tab:bandwidth_results}, $\mathrm{bw}^{r}$ denotes the bandwidth produced by rCM. The two variants irCM-iter and irCM-time correspond to irCM with a fixed maximum number of iterations (2000) and with a time limit set to twice the running time of Bi-PLMA, respectively. The results show that Bi-PLMA achieves the best overall balance between solution quality and computational efficiency among the compared methods. In particular, it substantially improves upon the initial rCM solutions on most instances, and it is generally both faster and more effective than Bi-RoTS, indicating that PLMA is better suited than Ro-TS for the sequence of related QAP subproblems arising in the bisection framework. Compared with irCM, Bi-PLMA is highly competitive in terms of bandwidth while remaining computationally efficient, especially on large-scale instances. These results demonstrate that combining the proposed PLMA solver with the bisection strategy provides an effective approach for large-scale bandwidth minimization.

\begin{table}[t]
	\caption{Bandwidth results on the benchmark set.}
	\label{tab:bandwidth_results}
	\centering
	\setlength{\tabcolsep}{2pt}
	\renewcommand{\arraystretch}{0.95}

	\begin{tabular}{
			l r r
			@{\hspace{0.6em}} r r
			@{\hspace{0.6em}} r r
			@{\hspace{0.6em}} r r
			@{\hspace{0.6em}} r r
		}
		\toprule
		&  &  & \multicolumn{2}{c}{Bi-PLMA} & \multicolumn{2}{c}{Bi-RoTS} & \multicolumn{2}{c}{irCM-iter} & \multicolumn{2}{c}{irCM-time} \\
		\cmidrule(lr){4-5}\cmidrule(lr){6-7}\cmidrule(lr){8-9}\cmidrule(lr){10-11}
		problem & n & $bw^r$ & bw & time & bw & time & bw & time & bw & time \\
		\midrule
		ash292 & 292 & 34 & \textbf{21} & 32.68 & 34 & 286.29 & \textbf{21} & 21.39 & \textbf{21} & 65.37 \\
		bcspwr04 & 274 & 58 & \textbf{27} & 40.27 & 58 & 235.75 & 29 & 165.33 & 29 & 80.61 \\
		bcsstk04 & 132 & 65 & \textbf{37} & 3.25 & 38 & 41.62 & 40 & 3.51 & 40 & 6.50 \\
		bcsstk05 & 153 & 24 & \textbf{20} & 3.90 & 23 & 182.75 & \textbf{20} & 3.18 & \textbf{20} & 7.81 \\
		bcsstk22 & 138 & 14 & 11 & 5.13 & 13 & 106.44 & \textbf{10} & 1.53 & \textbf{10} & 10.26 \\
		can\_144 & 144 & 18 & \textbf{13} & 9.50 & \textbf{13} & 120.26 & \textbf{13} & 5.52 & \textbf{13} & 18.99 \\
		can\_161 & 161 & 30 & \textbf{18} & 8.27 & 28 & 192.37 & 19 & 0.93 & \textbf{18} & 16.54 \\
		can\_187 & 187 & 23 & \textbf{13} & 6.78 & 16 & 366.42 & 15 & 1.37 & 15 & 13.57 \\
		can\_229 & 229 & 48 & \textbf{30} & 19.37 & 33 & 767.74 & 32 & 11.71 & 32 & 38.74 \\
		can\_256 & 256 & 118 & \textbf{59} & 52.78 & 82 & 944.16 & 62 & 189.03 & 62 & 105.55 \\
		can\_268 & 268 & 133 & \textbf{52} & 31.69 & 91 & 874.82 & 53 & 339.28 & 54 & 63.41 \\
		can\_292 & 292 & 67 & \textbf{39} & 29.50 & 67 & 285.73 & 50 & 40.47 & 50 & 59.01 \\
		dwt\_162 & 162 & 21 & \textbf{13} & 7.06 & 21 & 41.93 & 14 & 1.88 & 14 & 14.11 \\
		dwt\_193 & 193 & 59 & \textbf{32} & 13.64 & \textbf{32} & 356.45 & 34 & 57.31 & 34 & 27.33 \\
		dwt\_198 & 198 & 14 & \textbf{8} & 9.29 & 11 & 388.56 & 10 & 2.37 & 10 & 18.57 \\
		dwt\_209 & 209 & 59 & \textbf{24} & 26.12 & 29 & 571.81 & 29 & 9.56 & 29 & 52.24 \\
		dwt\_221 & 221 & 16 & \textbf{14} & 9.26 & 16 & 116.63 & \textbf{14} & 2.10 & \textbf{14} & 18.51 \\
		dwt\_234 & 234 & 25 & \textbf{11} & 21.29 & 19 & 691.36 & 12 & 10.03 & 12 & 42.59 \\
		dwt\_245 & 245 & 44 & \textbf{23} & 19.26 & 44 & 162.45 & 31 & 98.29 & 31 & 38.52 \\
		grid1 & 252 & 19 & \textbf{19} & 18.89 & \textbf{19} & 179.72 & \textbf{19} & 0.57 & \textbf{19} & 37.79 \\
		grid1\_dual & 224 & 17 & \textbf{17} & 16.54 & \textbf{17} & 122.17 & \textbf{17} & 0.48 & \textbf{17} & 33.09 \\
		jazz & 198 & 115 & \textbf{69} & 15.49 & 74 & 302.97 & 81 & 109.13 & 81 & 31.01 \\
		Journals & 124 & 117 & \textbf{98} & 5.66 & 99 & 95.63 & 102 & 4.17 & 102 & 11.32 \\
		lshp\_265 & 265 & 18 & \textbf{17} & 15.77 & 18 & 208.31 & \textbf{17} & 0.66 & \textbf{17} & 31.55 \\
		lund\_a & 147 & 23 & \textbf{23} & 3.58 & \textbf{23} & 30.26 & \textbf{23} & 0.31 & \textbf{23} & 7.15 \\
		lund\_b & 147 & 23 & \textbf{23} & 3.59 & \textbf{23} & 30.47 & \textbf{23} & 0.31 & \textbf{23} & 7.18 \\
		mesh3e1 & 289 & 17 & \textbf{17} & 9.87 & \textbf{17} & 278.60 & \textbf{17} & 0.72 & \textbf{17} & 19.73 \\
		mesh3em5 & 289 & 17 & \textbf{17} & 10.27 & \textbf{17} & 277.17 & \textbf{17} & 0.72 & \textbf{17} & 20.53 \\
		nos1 & 237 & 4 & 4 & 11.09 & 4 & 145.55 & \textbf{3} & 0.88 & \textbf{3} & 22.18 \\
		sphere3 & 258 & 34 & \textbf{27} & 47.23 & 34 & 191.50 & 28 & 1.39 & \textbf{27} & 94.46 \\
		Trefethen\_150 & 150 & 79 & \textbf{67} & 7.80 & \textbf{67} & 199.39 & 71 & 1.02 & 70 & 15.59 \\
		Trefethen\_200 & 200 & 100 & \textbf{85} & 24.75 & 88 & 559.94 & 88 & 1.21 & 88 & 49.49 \\
		Trefethen\_200b & 199 & 100 & \textbf{85} & 21.71 & 87 & 561.87 & 87 & 1.24 & 87 & 43.42 \\
		Trefethen\_300 & 300 & 152 & \textbf{128} & 71.07 & 152 & 312.64 & 132 & 3.27 & 132 & 142.14 \\
		\bottomrule
	\end{tabular}
\end{table}

\section{Conclusion}
    This paper presents PLMA, a permutation learning framework for the QAP that helps to narrow the gap between neural methods and strong handcrafted heuristics. The finetuning stage exploits the warm-start behavior of short MCMC chains so that each iteration of finetuning stays near previously discovered high quality assignments instead of restarting from scratch. A carefully designed additive energy-based model enables efficient exploration of the permutation space through constant-time swap evaluation, while the cross-graph attention mechanism offers a scalable and flexible way to capture the coupled two-graph structure in the QAP. Extensive experiments show that PLMA consistently delivers state-of-the-art results in both solution quality and computational efficiency across synthetic and real-world QAP benchmarks, including QAPLIB and the challenging Taixxeyy instances. These results suggest that PLMA remains effective across instances with diverse structural characteristics. We further show that PLMA extends naturally to QAP subproblems arising in the BM problem, highlighting its broader applicability. Future work will explore meta-learning strategies to better bridge pretraining and finetuning, and extend PLMA to broader classes of permutation-based problems with complex constraints.

\bibliography{main}
\bibliographystyle{IEEEtran}

\newpage
\clearpage
\onecolumn
\appendices
\renewcommand{\thesubsection}{\thesection.\arabic{subsection}}
\renewcommand{\thesubsectiondis}{\thesubsection}

\renewcommand{\thesubsubsection}{\thesubsection.\arabic{subsubsection}}
\renewcommand{\thesubsubsectiondis}{\thesubsubsection}

\section{Supplementary Theoretical Results}
\label{apx:proof}

\subsection{Proof of Theorem~\ref{thm:unbiased-gradient-estimator}}
\gradient*

\begin{proof}  \medskip\noindent We start by translating the target expectation to a convenient form. Recall that 
  \[
    p_\theta(\pi) \;=\; \frac{\exp(\Phi_\theta(\pi))}{Z_\theta}, \qquad Z_\theta = \sum_{\hat\pi\in\Pi_n}\exp(\Phi_\theta(\hat\pi)),
  \] 
  so we have $\log p_\theta(\pi) = \Phi_\theta(\pi) - \log Z_\theta$. Then for any $\pi$ drawn from $p_\theta$, we have 
  \begin{align*}
    &\mathbb{E}_{\pi \sim p_\theta}\!\big[g(\pi)\,\nabla_\theta\log p_\theta(\pi)\big] \\
      =\;& \mathbb{E}_{\pi \sim p_\theta}\!\Big[g(\pi)\,\big(\nabla_\theta \Phi_\theta(\pi) - \nabla_\theta \log Z_\theta\big)\Big] \\
      =\;& \mathbb{E}_{\pi \sim p_\theta}\!\Big[g(\pi)\,\nabla_\theta \Phi_\theta(\pi)\Big] \;-\; \nabla_\theta \log Z_\theta \; \mathbb{E}_{\pi \sim p_\theta}[g(\pi)] \\
      =\;& \mathbb{E}_{\pi \sim p_\theta}\!\Big[g(\pi)\,\nabla_\theta \Phi_\theta(\pi)\Big] \;\\
      & -\;\frac{1}{Z_\theta}\sum_{\hat\pi\in \Pi_n} \nabla_\theta \exp(\Phi_\theta(\hat\pi)) \;\mathbb{E}_{\pi \sim p_\theta}[g(\pi)] \\
      =\;& \mathbb{E}_{\pi \sim p_\theta}\!\Big[g(\pi)\,\nabla_\theta \Phi_\theta(\pi)\Big] \;\\
      &-\; \sum_{\hat\pi\in \Pi_n} \frac{\exp(\Phi_\theta(\hat\pi))}{Z_\theta}\nabla_\theta \Phi_\theta(\hat\pi) \;\mathbb{E}_{\pi \sim p_\theta}[g(\pi)] \\
      =\;& \mathbb{E}_{\pi \sim p_\theta}\!\Big[g(\pi)\,\nabla_\theta \Phi_\theta(\pi)\Big] \;-\;\mathbb{E}_{\pi \sim p_\theta}\!\big[\nabla_\theta \Phi_\theta(\pi)\big] \; \mathbb{E}_{\pi \sim p_\theta}[g(\pi)]  \\
      =\;& \mathbb{E}_{\pi \sim p_\theta}\!\Big[\big(g(\pi) - \mathbb{E}_{\hat\pi \sim p_\theta}[g(\hat\pi)]\big)\,\nabla_\theta \Phi_\theta(\pi)\Big]~,
  \end{align*}
  where in the second line the constant $\nabla_\theta \log Z_\theta$ is pulled out. The final expression shows that 
  \begin{align*}
      &\mathbb{E}_{\pi \sim p_\theta}\!\big[g(\pi)\,\nabla_\theta\log p_\theta(\pi)\big]\\
    \;=\; &\mathbb{E}_{\pi \sim p_\theta}\!\Big[(g(\pi) - \mathbb{E}_{p_\theta}[g(\pi)])\,\nabla_\theta \Phi_\theta(\pi)\Big]~.
  \end{align*}

  \textbf{(Unbiasedness)} Now let $\hat{G}_N$ denote the proposed estimator (\ref{eq:gd-est}) based on $N$ i.i.d. samples. Writing it out:
  \[
    \hat{G}_N \;:=\; \frac{1}{\,N-1\,}\sum_{i=1}^N \Big(g(\pi_i) - \frac{1}{N}\sum_{j=1}^N g(\pi_j)\Big)\;\nabla_\theta \Phi_\theta(\pi_i)~,
  \] 
  where $\pi_1,\dots,\pi_N \overset{\text{i.i.d.}}{\sim} p_\theta$. We will show $\mathbb{E}[\hat{G}_N]$ equals the target quantity above for any $N$, which establishes unbiasedness.

  Using linearity of expectation and the independence of samples, we can expand $\mathbb{E}[\hat{G}_N]$ as follows:
  \begin{align*}
    \mathbb{E}[\hat{G}_N] &= \frac{1}{\,N-1\,}\sum_{i=1}^N \mathbb{E}\Big[\big(g(\pi_i) - \frac{1}{N}\sum_{j=1}^N g(\pi_j)\big)\,\nabla_\theta \Phi_\theta(\pi_i)\Big] \\
    &= \frac{1}{\,N-1\,}\sum_{i=1}^N \Big( \mathbb{E}[g(\pi_i)\,\nabla_\theta \Phi_\theta(\pi_i)] \;-\; \mathbb{E}\Big[\frac{1}{N}\sum_{j=1}^N g(\pi_j)\,\nabla_\theta \Phi_\theta(\pi_i)\Big]\Big)~.
  \end{align*}
  For each fixed $i$, we evaluate the two expectations inside the sum:
  \begin{itemize}\setlength{\itemsep}{0pt}
    \item $\mathbb{E}[g(\pi_i)\,\nabla_\theta \Phi_\theta(\pi_i)] = \mathbb{E}_{\pi\sim p_\theta}[g(\pi)\,\nabla_\theta \Phi_\theta(\pi)]$, since $\pi_i$ has the same distribution as a fresh draw $\pi \sim p_\theta$.
    \item $\displaystyle \mathbb{E}\Big[\frac{1}{N}\sum_{j=1}^N g(\pi_j)\,\nabla_\theta \Phi_\theta(\pi_i)\Big] 
      = \frac{1}{N}\sum_{j=1}^N \mathbb{E}[\,g(\pi_j)\,\nabla_\theta \Phi_\theta(\pi_i)\,]$  \\
      $\quad= \frac{1}{N}\Big(\mathbb{E}[g(\pi_i)\,\nabla_\theta \Phi_\theta(\pi_i)] + \sum_{j\neq i}\mathbb{E}[g(\pi_j)\,\nabla_\theta \Phi_\theta(\pi_i)]\Big)$.
  \end{itemize}
  For $j \neq i$, $\pi_j$ is independent of $\pi_i$, so $\mathbb{E}[g(\pi_j)\,\nabla_\theta \Phi_\theta(\pi_i)] = \mathbb{E}[g(\pi_j)]\,\mathbb{E}[\nabla_\theta \Phi_\theta(\pi_i)] = \mathbb{E}_{p_\theta}[g(\pi)]\,\mathbb{E}_{p_\theta}[\nabla_\theta \Phi_\theta(\pi)]$. Therefore, continuing the calculation:
  \begin{align*}
    \mathbb{E}[\hat{G}_N] &= \frac{1}{\,N-1\,}\sum_{i=1}^N \Big( \mathbb{E}[g(\pi)\,\nabla_\theta \Phi_\theta(\pi)] \;-\; \frac{1}{N}\Big[\mathbb{E}[g(\pi)\,\nabla_\theta \Phi_\theta(\pi)] + (N-1)\,\mathbb{E}[g(\pi)]\,\mathbb{E}[\nabla_\theta \Phi_\theta(\pi)]\Big]\Big) \\
      &= \frac{1}{\,N-1\,}\sum_{i=1}^N \Big( \mathbb{E}[g(\pi)\,\nabla_\theta \Phi_\theta(\pi)] - \frac{1}{N}\mathbb{E}[g(\pi)\,\nabla_\theta \Phi_\theta(\pi)] - \frac{N-1}{N}\,\mathbb{E}[g(\pi)]\,\mathbb{E}[\nabla_\theta \Phi_\theta(\pi)]\Big) \\
      &= \frac{1}{\,N-1\,}\sum_{i=1}^N \frac{N-1}{N}\Big(\mathbb{E}[g(\pi)\,\nabla_\theta \Phi_\theta(\pi)] - \mathbb{E}[g(\pi)]\,\mathbb{E}[\nabla_\theta \Phi_\theta(\pi)]\Big) \\
      &= \mathbb{E}[g(\pi)\,\nabla_\theta \Phi_\theta(\pi)] \;-\; \mathbb{E}[g(\pi)]\,\mathbb{E}[\nabla_\theta \Phi_\theta(\pi)] \\
      &= \mathbb{E}_{\pi\sim p_\theta}\!\Big[(g(\pi) - \mathbb{E}_{p_\theta}[g(\pi)])\,\nabla_\theta \Phi_\theta(\pi)\Big]~.
  \end{align*}
  Recall that this quantity is exactly $\mathbb{E}_{\pi\sim p_\theta}[\,g(\pi)\,\nabla_\theta \log p_\theta(\pi)\,]$.  Therefore $\mathbb{E}[\hat{G}_N]$ equals the desired one. In other words, $\hat{G}_N$ is an \emph{unbiased} estimator of $\mathbb{E}_{\pi\sim p_\theta}[g(\pi)\,\nabla_\theta \log p_\theta(\pi)]$.

  \medskip\noindent\textbf{(Consistency)} Finally, we show that $\hat{G}_N$ converges to the true value as $N\to\infty$. Since $\Pi_n$ is a finite set, $g$ and $\Phi_\theta$ are obviously bounded and then $\mathbb E|g(\pi)|<+\infty$, $\mathbb E|\nabla \Phi_\theta(\pi)|<+\infty$ and  $\mathbb E|g(\pi)\nabla \Phi_\theta(\pi)|<+\infty$. By the Strong Law of Large Numbers (SLLN), the sample averages converge to their expectations almost surely. In particular, as $N\to\infty$ we have 
  \[
    \frac{1}{N}\sum_{i=1}^N g(\pi_i) \;\xrightarrow{\text{a.s.}}\; \mathbb{E}_{p_\theta}[g(\pi)], \qquad 
    \frac{1}{N}\sum_{i=1}^N \nabla_\theta \Phi_\theta(\pi_i) \;\xrightarrow{\text{a.s.}}\; \mathbb{E}_{p_\theta}[\nabla_\theta \Phi_\theta(\pi)],
  \] 
  and 
  \[
    \frac{1}{N}\sum_{i=1}^N g(\pi_i)\,\nabla_\theta \Phi_\theta(\pi_i) \;\xrightarrow{\text{a.s.}}\; \mathbb{E}_{p_\theta}[\,g(\pi)\,\nabla_\theta \Phi_\theta(\pi)\,]~,
  \] 
  provided the assumed moment conditions hold. Notice that we can rewrite the estimator in a form convenient for taking the limit:
  \[
    \hat{G}_N \;=\; \frac{N}{\,N-1\,}\Bigg(\frac{1}{N}\sum_{i=1}^N g(\pi_i)\,\nabla_\theta \Phi_\theta(\pi_i) \;-\; \Big(\frac{1}{N}\sum_{i=1}^N g(\pi_i)\Big)\Big(\frac{1}{N}\sum_{i=1}^N \nabla_\theta \Phi_\theta(\pi_i)\Big)\Bigg)~.
  \] 
  As $N\to\infty$, the prefactor $\frac{N}{N-1}\to 1$, and each of the big parentheses converges to the corresponding expectation. Hence, by the above limits and continuous mapping, $\hat{G}_N$ converges almost surely to 
  \[
    \mathbb{E}_{p_\theta}[\,g(\pi)\,\nabla_\theta \Phi_\theta(\pi)\,] - \mathbb{E}_{p_\theta}[g(\pi)]\,\mathbb{E}_{p_\theta}[\nabla_\theta \Phi_\theta(\pi)] \;=\; \mathbb{E}_{p_\theta}[\,g(\pi)\,\nabla_\theta \log p_\theta(\pi)\,]~,
  \] 
  which is the true parameter value we seek. This shows that $\hat{G}_N$ is a \emph{consistent} estimator. 
  
  In summary, $\hat{G}_N$ is unbiased and $\hat{G}_N \to \mathbb{E}_{p_\theta}[g(\pi)\,\nabla_\theta \log p_\theta(\pi)]$ almost surely as $N\to\infty$.
\end{proof}

\section{Supplementary Implementation Details}

\subsection{Dataset Details}
\label{subsec:datasets}
An instance of the Quadratic Assignment Problem (QAP) is defined by two $n \times n$ matrices: a distance matrix $D$ and a flow matrix $F$. We use both synthetic datasets and widely-used real-world benchmarks to thoroughly evaluate our model.

\textbf{Synthetic Datasets}
We generate two distinct classes of synthetic instances to assess performance on both structured and unstructured problems.
\begin{itemize}
    \item \textbf{Geometrically Structured instances (SAWT style):} Following the generation procedure from SAWT~\cite{tanlearning}, we create instances with geometric structure. The distance matrix $D$ is a 2D-Euclidean distance matrix derived from $n$ coordinates sampled uniformly from the unit square $[0,1]^2$. The flow matrix $F$ is initially sampled from $\mathcal{U}[0,1]$, made symmetric, and then sparsified by randomly setting 70\% of off-diagonal elements to zero.
    \item \textbf{Uniformly Random Instances (Tai-a style):} To assess performance on unstructured problems, we generate instances following the `Tai-a' specification from QAPLIB. The elements of both $D$ and $F$ are sampled independently and identically distributed (i.i.d.) from a uniform distribution $\mathcal{U}[0,1]$ and are then made symmetric.  
\end{itemize}

\textbf{Real-World Benchmarks}
We also test our model on two widely used and challenging benchmarks, QAPLIB and Taixxeyy.
\begin{itemize}
    \item \textbf{QAPLIB:} QAPLIB~\cite{burkard1997qaplib} is a widely used benchmark library for the Quadratic Assignment Problem (QAP). It contains 134 instances drawn from 15 categories, covering a broad range of structural characteristics and application backgrounds. Each instance is typically named in the format of “author name – problem size – index (optional)”. According to their generation mechanisms, the instances in QAPLIB can be broadly grouped into three classes: (i) unstructured random instances, in which both the flow matrix and the distance matrix are generated element-wise at random, such as the Tai-a and Rou series; (ii) grid-based distance instances, where the distance matrix is defined by the Manhattan or Euclidean distances among facilities located on a grid, as in the Sko, Nug, and Wil series; and (iii) real-life and real-like instances, which arise from practical applications or realistic modeling scenarios, such as keyboard design (Bur) and circuit layout (Esc). This diversity has made QAPLIB a standard benchmark for evaluating the performance and generalization ability of QAP algorithms.
    \item \textbf{Taixxeyy:} The Taixxeyy benchmark \cite{drezner2005taie} comprises a suite of QAP instances where `xx' in the name denotes the problem size and `yy' the instance number. These instances are specifically engineered to be difficult for heuristics that rely on transposition-based local search. They are generated with a recursive, hierarchical block structure, creating clusters of facilities with high intra-group flows and small intra-group distances, but low flows and large distances between groups.
    To make the problem challenging, this block structure is intentionally obscured with small inter-block flows and large finite distances. This creates a highly rugged solution landscape where meaningful improvement often requires accepting a series of deteriorating solutions.
\end{itemize}

\subsection{Baseline Details}\label{apx:subsec-implementation-baseline}
We evaluate our method against a range of baselines, from classic heuristics to modern learning-based approaches. We first distinguish between the two primary formulations of the Quadratic Assignment Problem (QAP). The general form, Lawler's QAP, is defined by a quadratic matrix $K$ with the following objective:
\begin{align*}
    &\min_{X} \quad \operatorname{vec}(X)^\top K \operatorname{vec}(X), \\
    &\text{s.t.} \quad X \in \{0,1\}^{n \times n}, \; X 1_{n} = 1_{n}, \; X^\top1_{n} = 1_{n}
\end{align*}
A widely-used special case is the Koopmans-Beckmann's QAP, where the cost is determined by flow ($F$) and distance ($D$) matrices, formulated as:
\begin{align*}
    &\min_{X} \quad \text{trace}(F^{\top}XDX^{\top}), \\
    &\text{s.t.} \quad X \in \{0,1\}^{n \times n}, \; X 1_{n} = 1_{n}, \; X^\top1_{n} = 1_{n}
\end{align*}
The Koopmans-Beckmann formulation is a special case of Lawler's QAP where $K = D^{\top} \otimes F$. Our method, PLMA, is designed specifically for this latter formulation.

\textbf{Ro-TS}~\cite{Taillard1991} is a highly optimized tabu search heuristic specifically for the Koopmans-Beckmann's QAP. We use the official C implementation from Éric Taillard’s website\footnote{\url{https://mistic.iict-heig-vd.ch/taillard/codes.dir/tabou_qap2.c}}.

\textbf{BMA}~\cite{BMA} is a powerful modern heuristic that integrates local search with evolutionary strategies. We use the official C implementation from Jin-Kao Hao's website\footnote{\url{https://leria-info.univ-angers.fr/~jinkao.hao/BMA.html}}.

\textbf{C-SA}~\cite{connolly1990SA} is a simulated annealing algorithm proposed by Connolly, featuring an improved annealing scheme for the QAP that performs well on a range of benchmark instances with modest computational effort. We use the C implementation from Éric Taillard’s website\footnote{\url{https://mistic.iict-heig-vd.ch/taillard/codes.dir/sa_qap.cpp}}.

\textbf{IPFP}~\cite{Leordeanu2005ipfp} is an iterative algorithm that seeks a good discrete solution by optimizing a continuous relaxation. While initially designed for Lawler's QAP, we use an improved version tailored for the Koopmans-Beckmann's QAP, as detailed in Algorithm~\ref{alg:ipfp}. To ensure comparable computation time, we run IPFP with a batch of random initializations and report the best solution found.

\begin{algorithm}[H]
\caption{IPFP for Koopmans-Beckmann's QAP}
\label{alg:ipfp}
\begin{algorithmic}[1]
\Require Initial permutation matrix $X_0$, number of max iterations $K$, convergence tolerance $\text{tol}$.
\State Initialize best solution $X^* \leftarrow X_0$.
\For{$k=0, \dots, K-1$}
    \State Calculate the gradient $\nabla f(X_k) = FX_kD^\top+F^\top X_kD$.
    \State Let $B_k \leftarrow \arg\min_{B \in \Pi_n} \langle B, \nabla f(X_k) \rangle$ via Hungarian algorithm.
    \State Let $a=\langle F,(B_k-X_k)D(B_k^{\top}-X_k^{\top})\rangle,\; b=\langle F,(B_k-X_k)DX_k^{\top}+X_kD(B_k^{\top}-X_k^{\top})\rangle$.
    \If{$a>0$}
        \State $t_k=\min(-\frac{b}{2a},1)$
    \Else
        \State $t_k=1$
    \EndIf
    
    \State Update $X_{k+1} \leftarrow X_k + t_k(B_k - X_k)$
    \If{$f(B_k) < f(X^*)$}
        \State $X^* \leftarrow B_k$
    \EndIf
    \If{$\|X_{k+1} - X_k\|_F \le \text{tol}$}
        \State \textbf{break}
    \EndIf
\EndFor
\State \Return The best found solution $X^*$.
\end{algorithmic}
\end{algorithm}

\textbf{SM}~\cite{LeordeanuH05sm} solves the Lawler's QAP by constructing an association graph and applying spectral methods to find the principal eigenvector, which is then discretized to obtain the final matching.

\textbf{RRWM}~\cite{ChoLL10rrwm} also addresses the Lawler's QAP. It models the matching problem as a random walk on the association graph, introducing a re-weighting scheme to better preserve matching constraints during the walk.

\textbf{SAWT}~\cite{tanlearning} is a reinforcement learning method that learns to improve solutions for the Koopmans-Beckmann's QAP. A key distinction of SAWT is that it requires 2D coordinates for the locations as input, which limits its direct applicability to datasets where the distance matrix $D$ is not derived from Euclidean distances. Since the authors did not provide a canonical way to handle this, we adapted the implementation for the QAPLIB dataset by using Multi-dimensional Scaling (MDS) for symmetric matrices and Principal Component Analysis (PCA) for asymmetric matrices to generate the required 2D coordinates. For large-scale instances, we trained a new model with a larger initialization size ($N\_ init=512$) to avoid the limitations of the provided pre-trained models ($N\_ init=128$).

\textbf{NGM}~\cite{wang2021neural}
approaches the problem by framing it as a vertex classification task on the association graph. We include NGM as a representative of learning-based methods that operate on the $n^2$ association graph, which, as noted in our main paper, can face scalability challenges.

\textbf{AR-Seq} uses an auto-regressive model and generates solutions by sequentially constructive sampling. Given a heatmap $\phi=\phi(\theta,\mathcal{P})\in\mathbb{R}^{n\times n}$, our method constructs an energy-based model. As an alternative, one can define an auto-regressive model $q_{\theta}$ as follows
\begin{equation*}
    q_{\theta}(\pi):= \prod_{i=1}^n\frac{\exp(\phi_{i,\pi(i)})}{\sum\limits_{j=i}^{n}\exp(\phi_{i,\pi(j)})}.
\end{equation*}
This factorization yields a sequence of conditional probabilities and therefore admits sequential constructive sampling. Owing to its ease of sampling, this formulation is widely used in prior active search methods. 

\section{Supplementary Experimental Details}
\subsection{Model Architecture Details}
\label{subsec:architecture_details}

Our cross-graph attention network consists of a series of graph neural network layers followed by cross-attention blocks. The specific architectural parameters are listed in Table~\ref{tab:architecture}.

\begin{table}[t]
\caption{Model architecture parameters for PLMA.}
\label{tab:architecture}
\centering
\begin{tabular}{lc}
\toprule
\textbf{Parameter} & \textbf{Value} \\
\midrule
Initial Feature Dimension ($d_{in}$) & 16 \\
Node Embedding Dimension ($d$) & 256 \\
GNN Layers ($l_1$) & 10 \\
Cross-Attention Blocks & 1 \\
Attention Heads per Block & 8 \\
Log-Sinkhorn Iterations & 1 \\
\bottomrule
\end{tabular}
\end{table}

\subsection{Training and Finetuning}
All models were trained using the PyTorch framework. The hyperparameters used for pre-training and warm-started MCMC finetuning are detailed in Table~\ref{tab:hyperparameters}. 

\begin{table}[t]
\caption{Hyperparameters for PLMA.}
\label{tab:hyperparameters}
\centering
\begin{tabular}{lcc}
\toprule
\textbf{Parameter} & \textbf{Pre-training} & \textbf{Finetuning} \\
\midrule
Optimizer & Adam & Adam \\
Learning Rate & 1e-4 & 1e-4 \\
Batch Size & 64 & 256 \\
Training Steps & 468,900 & -- \\
Total Finetuning Steps ($T$) & -- & 200 \\
Starting Points ($K$) & 400 & 20 \\
Chains per Point ($M$) & - & 20 \\
Chain Length ($L$) & $n$ & $\lfloor n/3 \rfloor$ \\
Local Search Iterations ($T_{LS}$) & 1 & $n$ \\
\bottomrule
\end{tabular}
\end{table}

\section{Supplementary Additional Results}
\begin{table}[!h]
\caption{Efficiency comparison on the uniformly random dataset with $n=100$. PLMA (instance-wise) runs the same finetuning independently on each test instance without batching.}
\label{tab:sequential}
\centering
\footnotesize
\setlength{\tabcolsep}{12pt}
\begin{tabular}{l c c c}
\toprule
Algorithm & Cost & Gap & Time \\
\midrule
Ro-TS (1k) & 2195.98 & 0.13\% & 38m59s \\
Ro-TS (5k) & 2193.16 & 0.00\% & 3h15m \\
\textbf{PLMA (instance-wise)} & 2193.96& 0.05\%& 31m15s\\
\textbf{PLMA} & 2193.13 & 0.00\% & 14m1s \\
\bottomrule
\end{tabular}
\end{table}

\subsection{Efficiency of Batched Finetuning}\label{apx:subsec-batch-wise}
In Table~\ref{tab:sequential}, we compare the efficiency of the batch-wise and instance-wise variants of PLMA on the uniformly random dataset with $n=100$ and 
200 finetuning steps. The batched configuration achieves a substantial reduction in runtime relative to the instance-wise variant. Despite sharing one set of network parameters across instances during updates, the batched variant preserves solution quality at the level of the strongest baseline. These observations support batched warm-started finetuning as an effective deployment-time adaptation strategy that improves computational efficiency without compromising final quality. Moreover, the instance-wise variant remains efficient in its own right and outperforms Ro-TS in runtime under comparable settings while delivering near-optimal solutions.

\subsection{Detailed Results on QAPLIB}
For a comprehensive evaluation, we present a detailed breakdown of our model's performance on each instance of the QAPLIB benchmark in Table \ref{tab:qaplib_details}. To account for the stochastic nature of the solver, each algorithm was executed for 10 independent runs. The table lists the minimum, mean, and maximum gaps and average computation time across the 10 runs.

\subsection{Detailed Results on Taixxeyy}
We also present the detailed results for the challenging Taixxeyy datasets. As noted by \cite{drezner2005taie},  the variance of the solutions can be very high for these instances, making it injudicious to rely solely on the average and standard deviation of solution values. In Table~\ref{tab:taixxeyy_details}, we present the minimum, mean, and maximum gaps and average computation time across the 10 runs for each method.
\clearpage
\onecolumn
{
\fontsize{7.5pt}{10.5pt}\selectfont
\setlength{\tabcolsep}{3pt}
\begin{longtable}{l cc cc cc cc cc}
\caption{Detailed QAPLIB results. Gap is reported as (min, mean, max) over 10 runs; time is the average runtime (s).} \label{tab:qaplib_details} \\
\toprule
\textbf{Problem} & \multicolumn{2}{c}{\textbf{Ro-TS}} & \multicolumn{2}{c}{\textbf{BMA}} & \multicolumn{2}{c}{\textbf{IPFP}} & \multicolumn{2}{c}{\textbf{SAWT}} & \multicolumn{2}{c}{\textbf{PLMA}} \\
\cmidrule(lr){2-3} \cmidrule(lr){4-5} \cmidrule(lr){6-7} \cmidrule(lr){8-9} \cmidrule(lr){10-11}
\textbf{Name} & Gap & Time & Gap & Time & Gap & Time & Gap & Time & Gap & Time \\
\midrule
\endfirsthead

\multicolumn{11}{c}
{{\tablename\ \thetable{} -- continued from previous page}} \\
\toprule
\textbf{Problem} & \multicolumn{2}{c}{\textbf{Ro-TS}} & \multicolumn{2}{c}{\textbf{BMA}} & \multicolumn{2}{c}{\textbf{IPFP}} & \multicolumn{2}{c}{\textbf{SAWT}} & \multicolumn{2}{c}{\textbf{PLMA}} \\
\cmidrule(lr){2-3} \cmidrule(lr){4-5} \cmidrule(lr){6-7} \cmidrule(lr){8-9} \cmidrule(lr){10-11}
\textbf{Name} & Gap & Time & Gap & Time & Gap & Time & Gap & Time & Gap & Time \\
\midrule
\endhead

\bottomrule
\endfoot

bur26a & (0.00, \textbf{0.00}, 0.00) & 0.16 & (0.00, \textbf{0.00}, 0.00) & 0.04 & (0.04, 0.10, 0.14) & 0.38 & (2.81, 3.74, 4.45) & 14.77 & (0.00, \textbf{0.00}, 0.00) & 0.08 \\
bur26b & (0.00, \textbf{0.00}, 0.00) & 0.17 & (0.00, \textbf{0.00}, 0.00) & 0.05 & (0.01, 0.06, 0.23) & 0.40 & (3.27, 3.81, 4.63) & 13.90 & (0.00, \textbf{0.00}, 0.00) & 0.08 \\
bur26c & (0.00, \textbf{0.00}, 0.00) & 0.06 & (0.00, \textbf{0.00}, 0.00) & 0.08 & (0.01, 0.03, 0.04) & 0.37 & (2.43, 4.08, 4.83) & 13.89 & (0.00, \textbf{0.00}, 0.00) & 0.08 \\
bur26d & (0.00, \textbf{0.00}, 0.00) & 0.16 & (0.00, \textbf{0.00}, 0.00) & 0.05 & (0.01, 0.03, 0.05) & 0.40 & (3.78, 4.41, 5.08) & 13.87 & (0.00, 0.00, 0.00) & 0.10 \\
bur26e & (0.00, \textbf{0.00}, 0.00) & 0.05 & (0.00, \textbf{0.00}, 0.00) & 0.06 & (0.01, 0.02, 0.04) & 0.23 & (2.73, 4.37, 5.48) & 13.90 & (0.00, \textbf{0.00}, 0.00) & 0.09 \\
bur26f & (0.00, \textbf{0.00}, 0.00) & 0.13 & (0.00, \textbf{0.00}, 0.00) & 0.03 & (0.02, 0.02, 0.03) & 0.33 & (4.03, 4.96, 6.21) & 13.89 & (0.00, \textbf{0.00}, 0.00) & 0.07 \\
bur26g & (0.00, \textbf{0.00}, 0.00) & 0.12 & (0.00, \textbf{0.00}, 0.00) & 0.06 & (0.04, 0.06, 0.10) & 0.25 & (2.95, 3.63, 4.55) & 13.88 & (0.00, 0.00, 0.00) & 0.08 \\
bur26h & (0.00, \textbf{0.00}, 0.00) & 0.12 & (0.00, \textbf{0.00}, 0.00) & 0.04 & (0.02, 0.04, 0.06) & 0.25 & (2.62, 3.91, 4.90) & 13.88 & (0.00, \textbf{0.00}, 0.00) & 0.08 \\
chr12a & (0.00, \textbf{0.00}, 0.00) & 0.00 & (0.00, \textbf{0.00}, 0.00) & 0.00 & (3.81, 8.62, 18.40) & 0.21 & (45.35, 88.99, 116.21) & 13.32 & (0.00, \textbf{0.00}, 0.00) & 0.05 \\
chr12b & (0.00, \textbf{0.00}, 0.00) & 0.00 & (0.00, \textbf{0.00}, 0.00) & 0.00 & (0.00, \textbf{0.00}, 0.00) & 0.19 & (43.73, 88.49, 126.24) & 13.32 & (0.00, \textbf{0.00}, 0.00) & 0.04 \\
chr12c & (0.00, \textbf{0.00}, 0.00) & 0.00 & (0.00, \textbf{0.00}, 0.00) & 0.01 & (3.68, 12.58, 19.90) & 0.23 & (29.69, 67.48, 89.71) & 13.23 & (0.00, \textbf{0.00}, 0.00) & 0.08 \\
chr15a & (0.00, \textbf{0.00}, 0.00) & 0.03 & (0.00, \textbf{0.00}, 0.00) & 0.01 & (4.81, 16.94, 37.45) & 0.25 & (121.79, 158.68, 188.20) & 13.29 & (0.00, \textbf{0.00}, 0.00) & 0.09 \\
chr15b & (0.00, \textbf{0.00}, 0.00) & 0.01 & (0.00, \textbf{0.00}, 0.00) & 0.01 & (7.96, 13.81, 20.35) & 0.16 & (142.35, 198.93, 251.41) & 13.20 & (0.00, \textbf{0.00}, 0.00) & 0.11 \\
chr15c & (0.00, \textbf{0.00}, 0.00) & 0.02 & (0.00, \textbf{0.00}, 0.00) & 0.01 & (16.06, 29.30, 41.37) & 0.26 & (127.27, 182.16, 229.00) & 13.22 & (0.00, \textbf{0.00}, 0.00) & 0.15 \\
chr18a & (0.00, \textbf{0.00}, 0.00) & 0.02 & (0.00, \textbf{0.00}, 0.00) & 0.05 & (0.18, 21.75, 30.10) & 0.23 & (208.27, 247.32, 288.23) & 13.48 & (0.00, \textbf{0.00}, 0.00) & 0.20 \\
chr18b & (0.00, \textbf{0.00}, 0.00) & 0.00 & (0.00, \textbf{0.00}, 0.00) & 0.00 & (1.17, 4.45, 10.17) & 0.31 & (74.97, 86.31, 90.74) & 13.42 & (0.00, \textbf{0.00}, 0.00) & 0.05 \\
chr20a & (0.00, 0.84, 1.46) & 0.26 & (0.00, 0.18, 1.82) & 0.10 & (6.39, 17.01, 23.45) & 0.24 & (137.04, 170.41, 198.91) & 13.47 & (0.00, \textbf{0.00}, 0.00) & 0.43 \\
chr20b & (0.00, 2.35, 4.87) & 0.29 & (2.00, 2.56, 3.05) & 0.24 & (11.05, 16.18, 19.41) & 0.27 & (156.31, 177.75, 187.12) & 13.49 & (0.00, \textbf{0.00}, 0.00) & 1.11 \\
chr20c & (0.00, \textbf{0.00}, 0.00) & 0.13 & (0.00, \textbf{0.00}, 0.00) & 0.11 & (0.00, 12.02, 21.37) & 0.13 & (164.32, 258.89, 322.98) & 13.50 & (0.00, \textbf{0.00}, 0.00) & 0.62 \\
chr22a & (0.00, 0.10, 0.32) & 0.33 & (0.00, \textbf{0.00}, 0.00) & 0.07 & (9.81, 12.36, 14.81) & 0.40 & (44.31, 56.17, 67.90) & 13.48 & (0.00, \textbf{0.00}, 0.00) & 0.32 \\
chr22b & (0.00, 0.69, 1.10) & 0.44 & (0.00, \textbf{0.00}, 0.00) & 0.08 & (12.98, 15.00, 17.79) & 0.38 & (46.04, 53.95, 60.80) & 13.46 & (0.00, \textbf{0.00}, 0.00) & 0.56 \\
chr25a & (0.00, 2.80, 5.69) & 0.63 & (1.84, 3.74, 6.38) & 0.40 & (18.07, 28.65, 44.47) & 0.39 & (198.16, 242.87, 266.49) & 13.55 & (0.00, \textbf{0.00}, 0.00) & 0.50 \\
els19 & (0.00, \textbf{0.00}, 0.00) & 0.02 & (0.00, \textbf{0.00}, 0.00) & 0.03 & (1.44, 10.18, 33.73) & 0.36 & (47.37, 47.37, 47.37) & 13.37 & (0.00, \textbf{0.00}, 0.00) & 0.09 \\
esc128 & (0.00, \textbf{0.00}, 0.00) & 10.30 & (0.00, \textbf{0.00}, 0.00) & 0.05 & (0.00, \textbf{0.00}, 0.00) & 4.63 & (125.00, 182.81, 206.25) & 21.46 & (0.00, \textbf{0.00}, 0.00) & 0.13 \\
esc16a & (0.00, \textbf{0.00}, 0.00) & 0.00 & (0.00, \textbf{0.00}, 0.00) & 0.00 & (0.00, \textbf{0.00}, 0.00) & 0.29 & (5.88, 17.06, 26.47) & 13.11 & (0.00, \textbf{0.00}, 0.00) & 0.05 \\
esc16b & (0.00, \textbf{0.00}, 0.00) & 0.00 & (0.00, \textbf{0.00}, 0.00) & 0.00 & (0.00, \textbf{0.00}, 0.00) & 0.37 & (0.00, 0.27, 0.68) & 13.10 & (0.00, \textbf{0.00}, 0.00) & 0.04 \\
esc16c & (0.00, \textbf{0.00}, 0.00) & 0.00 & (0.00, \textbf{0.00}, 0.00) & 0.00 & (0.00, \textbf{0.00}, 0.00) & 0.37 & (5.00, 14.50, 21.25) & 13.11 & (0.00, \textbf{0.00}, 0.00) & 0.04 \\
esc16d & (0.00, \textbf{0.00}, 0.00) & 0.00 & (0.00, \textbf{0.00}, 0.00) & 0.00 & (0.00, \textbf{0.00}, 0.00) & 0.28 & (25.00, 43.75, 62.50) & 13.10 & (0.00, \textbf{0.00}, 0.00) & 0.04 \\
esc16e & (0.00, \textbf{0.00}, 0.00) & 0.00 & (0.00, \textbf{0.00}, 0.00) & 0.00 & (0.00, \textbf{0.00}, 0.00) & 0.24 & (14.29, 17.14, 21.43) & 13.09 & (0.00, \textbf{0.00}, 0.00) & 0.04 \\
esc16g & (0.00, \textbf{0.00}, 0.00) & 0.00 & (0.00, \textbf{0.00}, 0.00) & 0.00 & (0.00, \textbf{0.00}, 0.00) & 0.19 & (7.69, 14.62, 15.38) & 13.08 & (0.00, \textbf{0.00}, 0.00) & 0.04 \\
esc16h & (0.00, \textbf{0.00}, 0.00) & 0.00 & (0.00, \textbf{0.00}, 0.00) & 0.00 & (0.00, \textbf{0.00}, 0.00) & 0.24 & (1.61, 6.12, 8.84) & 13.10 & (0.00, \textbf{0.00}, 0.00) & 0.04 \\
esc16i & (0.00, \textbf{0.00}, 0.00) & 0.00 & (0.00, \textbf{0.00}, 0.00) & 0.00 & (0.00, \textbf{0.00}, 0.00) & 0.19 & (28.57, 65.71, 114.29) & 13.09 & (0.00, \textbf{0.00}, 0.00) & 0.04 \\
esc16j & (0.00, \textbf{0.00}, 0.00) & 0.00 & (0.00, \textbf{0.00}, 0.00) & 0.00 & (0.00, \textbf{0.00}, 0.00) & 0.09 & (0.00, 37.50, 50.00) & 13.08 & (0.00, \textbf{0.00}, 0.00) & 0.04 \\
esc32a & (0.00, \textbf{0.00}, 0.00) & 0.26 & (0.00, \textbf{0.00}, 0.00) & 0.02 & (3.08, 5.38, 7.69) & 0.42 & (138.46, 157.08, 173.85) & 13.74 & (0.00, \textbf{0.00}, 0.00) & 0.33 \\
esc32b & (0.00, \textbf{0.00}, 0.00) & 0.07 & (0.00, \textbf{0.00}, 0.00) & 0.00 & (0.00, 0.95, 9.52) & 0.22 & (90.48, 90.48, 90.48) & 13.71 & (0.00, \textbf{0.00}, 0.00) & 0.09 \\
esc32c & (0.00, \textbf{0.00}, 0.00) & 0.00 & (0.00, \textbf{0.00}, 0.00) & 0.00 & (0.00, \textbf{0.00}, 0.00) & 0.54 & (7.17, 10.09, 12.15) & 13.73 & (0.00, \textbf{0.00}, 0.00) & 0.05 \\
esc32d & (0.00, \textbf{0.00}, 0.00) & 0.02 & (0.00, \textbf{0.00}, 0.00) & 0.00 & (0.00, \textbf{0.00}, 0.00) & 0.41 & (28.00, 33.50, 43.00) & 13.73 & (0.00, \textbf{0.00}, 0.00) & 0.05 \\
esc32e & (0.00, \textbf{0.00}, 0.00) & 0.00 & (0.00, \textbf{0.00}, 0.00) & 0.00 & (0.00, \textbf{0.00}, 0.00) & 0.30 & (0.00, \textbf{0.00}, 0.00) & 13.73 & (0.00, \textbf{0.00}, 0.00) & 0.05 \\
esc32g & (0.00, \textbf{0.00}, 0.00) & 0.00 & (0.00, \textbf{0.00}, 0.00) & 0.00 & (0.00, \textbf{0.00}, 0.00) & 0.25 & (0.00, 16.67, 66.67) & 13.72 & (0.00, \textbf{0.00}, 0.00) & 0.05 \\
esc32h & (0.00, \textbf{0.00}, 0.00) & 0.05 & (0.00, \textbf{0.00}, 0.00) & 0.02 & (0.46, 0.73, 0.91) & 0.60 & (19.18, 21.87, 26.03) & 13.76 & (0.00, \textbf{0.00}, 0.00) & 0.05 \\
esc64a & (0.00, \textbf{0.00}, 0.00) & 0.42 & (0.00, \textbf{0.00}, 0.00) & 0.01 & (0.00, \textbf{0.00}, 0.00) & 1.69 & (43.10, 74.14, 101.72) & 19.00 & (0.00, \textbf{0.00}, 0.00) & 0.07 \\
had12 & (0.00, \textbf{0.00}, 0.00) & 0.00 & (0.00, \textbf{0.00}, 0.00) & 0.00 & (0.24, 0.27, 0.36) & 0.33 & (3.51, 4.78, 6.17) & 13.02 & (0.00, \textbf{0.00}, 0.00) & 0.04 \\
had14 & (0.00, \textbf{0.00}, 0.00) & 0.00 & (0.00, \textbf{0.00}, 0.00) & 0.00 & (0.00, \textbf{0.00}, 0.00) & 0.33 & (3.16, 5.90, 7.56) & 13.07 & (0.00, \textbf{0.00}, 0.00) & 0.05 \\
had16 & (0.00, \textbf{0.00}, 0.00) & 0.01 & (0.00, \textbf{0.00}, 0.00) & 0.00 & (0.05, 0.06, 0.11) & 0.37 & (3.82, 5.61, 6.51) & 13.12 & (0.00, \textbf{0.00}, 0.00) & 0.04 \\
had18 & (0.00, \textbf{0.00}, 0.00) & 0.00 & (0.00, \textbf{0.00}, 0.00) & 0.01 & (0.00, 0.01, 0.04) & 0.46 & (4.07, 5.16, 6.20) & 13.36 & (0.00, \textbf{0.00}, 0.00) & 0.05 \\
had20 & (0.00, \textbf{0.00}, 0.00) & 0.01 & (0.00, \textbf{0.00}, 0.00) & 0.00 & (0.03, 0.05, 0.06) & 0.45 & (4.13, 5.44, 7.22) & 13.44 & (0.00, \textbf{0.00}, 0.00) & 0.06 \\
kra30a & (0.00, \textbf{0.00}, 0.00) & 0.37 & (0.00, \textbf{0.00}, 0.00) & 0.13 & (0.00, 0.62, 1.77) & 0.65 & (27.55, 33.24, 36.31) & 13.66 & (0.00, \textbf{0.00}, 0.00) & 0.12 \\
kra30b & (0.00, \textbf{0.00}, 0.00) & 0.29 & (0.00, \textbf{0.00}, 0.00) & 0.06 & (0.00, 0.21, 0.72) & 0.64 & (27.02, 32.08, 35.05) & 13.69 & (0.00, \textbf{0.00}, 0.00) & 0.67 \\
kra32 & (0.00, \textbf{0.00}, 0.00) & 0.08 & (0.00, \textbf{0.00}, 0.00) & 0.01 & (0.00, 1.12, 1.76) & 0.57 & (31.87, 34.77, 36.83) & 13.74 & (0.00, \textbf{0.00}, 0.00) & 0.14 \\
lipa20a & (0.00, \textbf{0.00}, 0.00) & 0.01 & (0.00, \textbf{0.00}, 0.00) & 0.01 & (0.00, 0.70, 2.20) & 0.71 & (4.51, 4.90, 5.21) & 13.39 & (0.00, \textbf{0.00}, 0.00) & 0.08 \\
lipa20b & (0.00, \textbf{0.00}, 0.00) & 0.00 & (0.00, \textbf{0.00}, 0.00) & 0.00 & (0.00, \textbf{0.00}, 0.00) & 0.36 & (0.00, \textbf{0.00}, 0.00) & 13.41 & (0.00, \textbf{0.00}, 0.00) & 0.05 \\
lipa30a & (0.00, \textbf{0.00}, 0.00) & 0.03 & (0.00, \textbf{0.00}, 0.00) & 0.05 & (0.18, 1.69, 2.00) & 1.27 & (3.56, 3.72, 3.86) & 13.69 & (0.00, \textbf{0.00}, 0.00) & 0.16 \\
lipa30b & (0.00, \textbf{0.00}, 0.00) & 0.00 & (0.00, \textbf{0.00}, 0.00) & 0.00 & (0.00, \textbf{0.00}, 0.00) & 0.69 & (0.00, \textbf{0.00}, 0.00) & 13.68 & (0.00, \textbf{0.00}, 0.00) & 0.07 \\
lipa40a & (0.00, \textbf{0.00}, 0.00) & 0.09 & (0.00, \textbf{0.00}, 0.00) & 0.04 & (1.53, 1.59, 1.66) & 1.99 & (2.90, 3.02, 3.11) & 13.98 & (0.00, \textbf{0.00}, 0.00) & 0.39 \\
lipa40b & (0.00, \textbf{0.00}, 0.00) & 0.01 & (0.00, \textbf{0.00}, 0.00) & 0.01 & (0.00, \textbf{0.00}, 0.00) & 1.00 & (0.00, \textbf{0.00}, 0.00) & 13.99 & (0.00, \textbf{0.00}, 0.00) & 0.09 \\
lipa50a & (0.00, \textbf{0.00}, 0.00) & 0.32 & (0.00, \textbf{0.00}, 0.00) & 0.09 & (1.36, 1.41, 1.44) & 4.32 & (2.40, 2.62, 2.72) & 14.25 & (0.00, \textbf{0.00}, 0.00) & 0.77 \\
lipa50b & (0.00, \textbf{0.00}, 0.00) & 0.04 & (0.00, \textbf{0.00}, 0.00) & 0.01 & (0.00, \textbf{0.00}, 0.00) & 2.33 & (0.00, \textbf{0.00}, 0.00) & 14.28 & (0.00, \textbf{0.00}, 0.00) & 0.09 \\
lipa60a & (0.00, \textbf{0.00}, 0.00) & 0.63 & (0.00, \textbf{0.00}, 0.00) & 0.94 & (1.11, 1.18, 1.23) & 6.26 & (2.15, 2.29, 2.35) & 18.78 & (0.00, \textbf{0.00}, 0.00) & 1.62 \\
lipa60b & (0.00, \textbf{0.00}, 0.00) & 0.16 & (0.00, \textbf{0.00}, 0.00) & 0.06 & (0.00, \textbf{0.00}, 0.00) & 3.34 & (0.00, \textbf{0.00}, 0.00) & 18.75 & (0.00, \textbf{0.00}, 0.00) & 0.15 \\
lipa70a & (0.00, \textbf{0.00}, 0.00) & 2.67 & (0.00, 0.06, 0.56) & 4.02 & (1.06, 1.07, 1.11) & 8.87 & (2.04, 2.08, 2.12) & 19.19 & (0.00, \textbf{0.00}, 0.00) & 2.01 \\
lipa70b & (0.00, \textbf{0.00}, 0.00) & 0.10 & (0.00, \textbf{0.00}, 0.00) & 0.35 & (0.00, 3.74, 18.72) & 4.73 & (0.00, \textbf{0.00}, 0.00) & 19.17 & (0.00, \textbf{0.00}, 0.00) & 0.17 \\
lipa80a & (0.00, \textbf{0.10}, 0.49) & 15.65 & (0.00, 0.16, 0.54) & 9.23 & (0.95, 0.98, 1.01) & 12.77 & (1.79, 1.85, 1.90) & 19.57 & (0.66, 0.69, 0.71) & 9.80 \\
lipa80b & (0.00, \textbf{0.00}, 0.00) & 0.62 & (0.00, \textbf{0.00}, 0.00) & 0.37 & (0.00, 1.97, 19.69) & 5.93 & (0.00, \textbf{0.00}, 0.00) & 19.57 & (0.00, \textbf{0.00}, 0.00) & 0.24 \\
lipa90a & (0.00, \textbf{0.38}, 0.49) & 31.28 & (0.00, 0.38, 0.50) & 13.95 & (0.86, 0.89, 0.92) & 16.70 & (1.64, 1.68, 1.72) & 19.97 & (0.63, 0.66, 0.68) & 9.15 \\
lipa90b & (0.00, \textbf{0.00}, 0.00) & 0.91 & (0.00, \textbf{0.00}, 0.00) & 0.45 & (0.00, 1.99, 19.90) & 7.68 & (0.00, \textbf{0.00}, 0.00) & 19.95 & (0.00, \textbf{0.00}, 0.00) & 0.25 \\
nug12 & (0.00, \textbf{0.00}, 0.00) & 0.00 & (0.00, \textbf{0.00}, 0.00) & 0.00 & (0.00, \textbf{0.00}, 0.00) & 0.24 & (10.73, 15.47, 18.34) & 13.01 & (0.00, \textbf{0.00}, 0.00) & 0.04 \\
nug14 & (0.00, \textbf{0.00}, 0.00) & 0.00 & (0.00, \textbf{0.00}, 0.00) & 0.01 & (0.00, 0.02, 0.20) & 0.30 & (11.44, 14.00, 15.58) & 13.07 & (0.00, \textbf{0.00}, 0.00) & 0.05 \\
nug15 & (0.00, \textbf{0.00}, 0.00) & 0.00 & (0.00, \textbf{0.00}, 0.00) & 0.00 & (0.00, \textbf{0.00}, 0.00) & 0.31 & (12.52, 16.02, 19.48) & 13.07 & (0.00, \textbf{0.00}, 0.00) & 0.04 \\
nug16a & (0.00, \textbf{0.00}, 0.00) & 0.00 & (0.00, \textbf{0.00}, 0.00) & 0.00 & (0.00, 0.09, 0.75) & 0.30 & (9.81, 15.24, 18.63) & 13.09 & (0.00, \textbf{0.00}, 0.00) & 0.07 \\
nug16b & (0.00, \textbf{0.00}, 0.00) & 0.00 & (0.00, \textbf{0.00}, 0.00) & 0.00 & (0.00, \textbf{0.00}, 0.00) & 0.30 & (13.71, 18.50, 21.94) & 13.11 & (0.00, \textbf{0.00}, 0.00) & 0.04 \\
nug17 & (0.00, \textbf{0.00}, 0.00) & 0.01 & (0.00, \textbf{0.00}, 0.00) & 0.04 & (0.00, 0.02, 0.12) & 0.33 & (13.63, 16.57, 19.52) & 13.30 & (0.00, \textbf{0.00}, 0.00) & 0.09 \\
nug18 & (0.00, \textbf{0.00}, 0.00) & 0.01 & (0.00, \textbf{0.00}, 0.00) & 0.02 & (0.00, 0.04, 0.41) & 0.33 & (13.68, 16.46, 19.79) & 13.37 & (0.00, \textbf{0.00}, 0.00) & 0.10 \\
nug20 & (0.00, \textbf{0.00}, 0.00) & 0.00 & (0.00, \textbf{0.00}, 0.00) & 0.00 & (0.00, 0.03, 0.31) & 0.37 & (14.24, 16.86, 19.22) & 13.42 & (0.00, \textbf{0.00}, 0.00) & 0.08 \\
nug21 & (0.00, \textbf{0.00}, 0.00) & 0.01 & (0.00, \textbf{0.00}, 0.00) & 0.01 & (0.00, 0.19, 0.41) & 0.39 & (19.85, 22.46, 25.43) & 13.42 & (0.00, \textbf{0.00}, 0.00) & 0.09 \\
nug22 & (0.00, \textbf{0.00}, 0.00) & 0.01 & (0.00, \textbf{0.00}, 0.00) & 0.02 & (0.00, \textbf{0.00}, 0.00) & 0.38 & (18.35, 21.69, 24.25) & 13.49 & (0.00, \textbf{0.00}, 0.00) & 0.08 \\
nug24 & (0.00, \textbf{0.00}, 0.00) & 0.02 & (0.00, \textbf{0.00}, 0.00) & 0.00 & (0.00, 0.01, 0.11) & 0.43 & (18.92, 22.00, 24.43) & 13.50 & (0.00, \textbf{0.00}, 0.00) & 0.11 \\
nug25 & (0.00, \textbf{0.00}, 0.00) & 0.01 & (0.00, \textbf{0.00}, 0.00) & 0.00 & (0.11, 0.11, 0.16) & 0.47 & (17.84, 19.49, 23.24) & 13.52 & (0.00, \textbf{0.00}, 0.00) & 0.21 \\
nug27 & (0.00, \textbf{0.00}, 0.00) & 0.03 & (0.00, \textbf{0.00}, 0.00) & 0.02 & (0.04, 0.04, 0.04) & 0.51 & (20.10, 22.20, 24.00) & 13.58 & (0.00, \textbf{0.00}, 0.00) & 0.11 \\
nug28 & (0.00, \textbf{0.00}, 0.00) & 0.13 & (0.00, \textbf{0.00}, 0.00) & 0.03 & (0.00, 0.15, 0.39) & 0.52 & (19.09, 22.02, 23.50) & 13.60 & (0.00, \textbf{0.00}, 0.00) & 0.16 \\
nug30 & (0.00, \textbf{0.00}, 0.00) & 0.12 & (0.00, \textbf{0.00}, 0.00) & 0.16 & (0.00, 0.04, 0.07) & 0.60 & (20.08, 21.86, 23.74) & 13.69 & (0.00, 0.01, 0.07) & 1.21 \\
rou12 & (0.00, \textbf{0.00}, 0.00) & 0.00 & (0.00, \textbf{0.00}, 0.00) & 0.00 & (0.00, 0.38, 1.45) & 0.29 & (9.88, 12.06, 14.99) & 13.00 & (0.00, \textbf{0.00}, 0.00) & 0.05 \\
rou15 & (0.00, \textbf{0.00}, 0.00) & 0.00 & (0.00, \textbf{0.00}, 0.00) & 0.00 & (0.00, 1.20, 3.19) & 0.30 & (12.25, 16.09, 18.39) & 13.09 & (0.00, \textbf{0.00}, 0.00) & 0.07 \\
rou20 & (0.00, \textbf{0.00}, 0.00) & 0.12 & (0.00, 0.01, 0.01) & 0.16 & (0.40, 0.73, 1.16) & 0.39 & (12.21, 14.94, 16.64) & 13.40 & (0.00, \textbf{0.00}, 0.00) & 0.12 \\
scr12 & (0.00, \textbf{0.00}, 0.00) & 0.00 & (0.00, \textbf{0.00}, 0.00) & 0.00 & (0.00, 1.35, 3.87) & 0.23 & (15.47, 19.07, 26.54) & 12.98 & (0.00, \textbf{0.00}, 0.00) & 0.04 \\
scr15 & (0.00, \textbf{0.00}, 0.00) & 0.00 & (0.00, \textbf{0.00}, 0.00) & 0.00 & (0.00, 0.94, 4.12) & 0.26 & (30.38, 34.71, 42.59) & 13.06 & (0.00, \textbf{0.00}, 0.00) & 0.06 \\
scr20 & (0.00, \textbf{0.00}, 0.00) & 0.02 & (0.00, \textbf{0.00}, 0.00) & 0.01 & (0.03, 1.44, 2.85) & 0.34 & (38.23, 50.37, 56.50) & 13.37 & (0.00, \textbf{0.00}, 0.00) & 0.13 \\
sko100a & (0.03, 0.09, 0.14) & 47.00 & (0.02, \textbf{0.06}, 0.12) & 18.87 & (0.22, 0.32, 0.42) & 12.58 & (13.96, 14.71, 15.64) & 19.51 & (0.04, 0.07, 0.10) & 9.46 \\
sko100b & (0.01, 0.05, 0.08) & 47.02 & (0.01, \textbf{0.04}, 0.17) & 18.23 & (0.12, 0.33, 0.53) & 12.68 & (13.59, 14.54, 15.12) & 19.44 & (0.00, 0.06, 0.13) & 9.18 \\
sko100c & (0.01, 0.03, 0.07) & 47.00 & (0.00, 0.03, 0.11) & 18.94 & (0.09, 0.24, 0.52) & 13.04 & (14.43, 15.28, 15.58) & 19.42 & (0.00, \textbf{0.02}, 0.04) & 9.44 \\
sko100d & (0.05, 0.08, 0.12) & 47.07 & (0.01, 0.07, 0.11) & 19.76 & (0.20, 0.33, 0.45) & 12.94 & (13.99, 14.69, 15.24) & 19.44 & (0.00, \textbf{0.06}, 0.12) & 9.53 \\
sko100e & (0.01, 0.05, 0.08) & 47.16 & (0.01, \textbf{0.03}, 0.05) & 18.58 & (0.12, 0.40, 0.56) & 12.85 & (14.34, 15.40, 15.89) & 19.37 & (0.02, 0.03, 0.07) & 9.55 \\
sko100f & (0.01, \textbf{0.06}, 0.13) & 47.10 & (0.02, 0.10, 0.25) & 18.58 & (0.39, 0.47, 0.56) & 12.80 & (13.57, 14.33, 14.76) & 19.37 & (0.05, 0.08, 0.14) & 9.49 \\
sko42 & (0.00, 0.00, 0.03) & 1.38 & (0.00, \textbf{0.00}, 0.00) & 0.19 & (0.04, 0.18, 0.35) & 1.55 & (18.52, 19.57, 20.72) & 13.99 & (0.00, \textbf{0.00}, 0.00) & 0.84 \\
sko49 & (0.00, 0.04, 0.07) & 4.81 & (0.00, 0.03, 0.07) & 3.65 & (0.07, 0.27, 0.56) & 2.10 & (17.15, 18.45, 19.86) & 14.19 & (0.00, \textbf{0.00}, 0.00) & 1.19 \\
sko56 & (0.00, 0.03, 0.19) & 6.32 & (0.00, \textbf{0.00}, 0.01) & 4.36 & (0.09, 0.29, 0.57) & 2.72 & (17.12, 18.27, 19.33) & 17.71 & (0.00, \textbf{0.00}, 0.01) & 4.21 \\
sko64 & (0.00, 0.02, 0.06) & 10.68 & (0.00, 0.00, 0.01) & 4.07 & (0.02, 0.23, 0.39) & 3.72 & (16.03, 16.90, 18.06) & 17.96 & (0.00, \textbf{0.00}, 0.01) & 3.79 \\
sko72 & (0.02, 0.07, 0.13) & 16.84 & (0.00, 0.04, 0.14) & 7.58 & (0.17, 0.28, 0.55) & 5.08 & (16.04, 16.46, 17.28) & 18.30 & (0.00, \textbf{0.01}, 0.03) & 6.95 \\
sko81 & (0.01, 0.06, 0.11) & 24.19 & (0.01, 0.04, 0.07) & 10.19 & (0.20, 0.31, 0.42) & 7.57 & (15.28, 15.72, 16.17) & 18.63 & (0.00, \textbf{0.04}, 0.07) & 8.61 \\
sko90 & (0.04, 0.08, 0.16) & 33.29 & (0.01, 0.07, 0.21) & 12.42 & (0.21, 0.31, 0.42) & 9.59 & (14.74, 15.19, 15.54) & 18.98 & (0.00, \textbf{0.06}, 0.16) & 8.87 \\
ste36a & (0.00, 0.02, 0.10) & 0.90 & (0.00, \textbf{0.00}, 0.00) & 0.56 & (0.82, 1.88, 3.99) & 0.75 & (56.02, 63.33, 64.52) & 13.83 & (0.00, \textbf{0.00}, 0.00) & 0.95 \\
ste36b & (0.00, \textbf{0.00}, 0.00) & 0.18 & (0.00, \textbf{0.00}, 0.00) & 0.11 & (0.56, 1.51, 3.13) & 0.48 & (172.62, 190.62, 206.91) & 13.83 & (0.00, \textbf{0.00}, 0.00) & 0.24 \\
ste36c & (0.00, \textbf{0.00}, 0.00) & 0.52 & (0.00, \textbf{0.00}, 0.00) & 0.47 & (0.71, 2.05, 3.15) & 0.86 & (51.21, 62.17, 68.23) & 13.84 & (0.00, \textbf{0.00}, 0.00) & 0.45 \\
tai100a & (0.91, \textbf{0.96}, 1.00) & 47.14 & (0.87, 1.08, 1.23) & 24.12 & (1.22, 1.44, 1.56) & 8.22 & (12.27, 12.66, 12.87) & 19.45 & (0.94, 1.12, 1.45) & 7.18 \\
tai100b & (0.00, 0.22, 0.42) & 47.09 & (0.02, 0.22, 0.53) & 16.15 & (0.25, 0.64, 0.84) & 12.24 & (35.90, 38.97, 41.36) & 19.41 & (0.00, \textbf{0.01}, 0.04) & 8.00 \\
tai10a & (0.00, \textbf{0.00}, 0.00) & 0.00 & (0.00, \textbf{0.00}, 0.00) & 0.00 & (0.00, 0.24, 0.59) & 0.24 & (2.30, 10.69, 16.82) & 12.90 & (0.00, \textbf{0.00}, 0.00) & 0.04 \\
tai12a & (0.00, \textbf{0.00}, 0.00) & 0.00 & (0.00, \textbf{0.00}, 0.00) & 0.00 & (0.00, \textbf{0.00}, 0.00) & 0.27 & (10.83, 14.80, 18.07) & 12.99 & (0.00, \textbf{0.00}, 0.00) & 0.04 \\
tai12b & (0.00, \textbf{0.00}, 0.00) & 0.00 & (0.00, \textbf{0.00}, 0.00) & 0.00 & (0.00, 0.32, 1.52) & 0.20 & (11.78, 21.86, 34.97) & 13.01 & (0.00, \textbf{0.00}, 0.00) & 0.04 \\
tai150b & (0.28, 0.45, 0.61) & 113.66 & (0.33, 0.57, 0.86) & 39.96 & (1.12, 1.29, 1.62) & 32.63 & (25.47, 26.25, 26.69) & 21.34 & (0.02, \textbf{0.22}, 0.48) & 12.76 \\
tai15a & (0.00, \textbf{0.00}, 0.00) & 0.00 & (0.00, \textbf{0.00}, 0.00) & 0.00 & (0.20, 0.60, 1.04) & 0.33 & (9.74, 11.41, 13.86) & 13.05 & (0.00, \textbf{0.00}, 0.00) & 0.05 \\
tai15b & (0.00, \textbf{0.00}, 0.00) & 0.00 & (0.00, \textbf{0.00}, 0.00) & 0.01 & (0.04, 0.20, 0.32) & 0.31 & (1.71, 2.39, 2.81) & 13.07 & (0.00, \textbf{0.00}, 0.00) & 0.04 \\
tai17a & (0.00, \textbf{0.00}, 0.00) & 0.02 & (0.00, \textbf{0.00}, 0.00) & 0.00 & (0.54, 1.28, 2.07) & 0.37 & (12.79, 15.02, 16.55) & 13.30 & (0.00, \textbf{0.00}, 0.00) & 0.23 \\
tai20a & (0.00, \textbf{0.00}, 0.00) & 0.10 & (0.00, 0.28, 0.47) & 0.28 & (0.00, 0.91, 1.82) & 0.37 & (14.75, 16.50, 18.31) & 13.38 & (0.00, 0.21, 0.57) & 5.25 \\
tai20b & (0.00, \textbf{0.00}, 0.00) & 0.01 & (0.00, \textbf{0.00}, 0.00) & 0.01 & (0.00, 0.31, 0.76) & 0.32 & (16.24, 34.75, 51.28) & 13.38 & (0.00, \textbf{0.00}, 0.00) & 0.08 \\
tai256c & (0.19, 0.21, 0.24) & 387.57 & (0.16, \textbf{0.19}, 0.22) & 299.88 & (0.67, 0.80, 1.00) & 9.69 & (62.32, 90.52, 115.80) & 25.31 & (0.17, 0.21, 0.24) & 34.86 \\
tai25a & (0.00, 0.11, 0.41) & 0.44 & (0.00, \textbf{0.00}, 0.00) & 0.11 & (0.88, 1.63, 2.32) & 0.52 & (15.02, 15.97, 16.91) & 13.52 & (0.00, \textbf{0.00}, 0.00) & 0.45 \\
tai25b & (0.00, \textbf{0.00}, 0.00) & 0.05 & (0.00, \textbf{0.00}, 0.00) & 0.04 & (0.09, 0.65, 1.39) & 0.37 & (36.75, 56.55, 71.64) & 13.53 & (0.00, \textbf{0.00}, 0.00) & 0.09 \\
tai30a & (0.00, 0.00, 0.02) & 0.80 & (0.00, 0.04, 0.18) & 0.46 & (0.95, 1.47, 1.69) & 0.66 & (14.18, 14.76, 15.27) & 13.68 & (0.00, \textbf{0.00}, 0.00) & 2.37 \\
tai30b & (0.00, 0.00, 0.00) & 0.65 & (0.00, \textbf{0.00}, 0.00) & 0.12 & (0.00, 0.56, 2.22) & 0.54 & (44.16, 52.71, 62.29) & 13.67 & (0.00, \textbf{0.00}, 0.00) & 1.10 \\
tai35a & (0.00, 0.29, 0.61) & 1.69 & (0.00, \textbf{0.16}, 0.58) & 0.81 & (1.36, 1.72, 2.20) & 0.80 & (14.15, 15.26, 15.89) & 13.80 & (0.00, 0.26, 0.49) & 9.61 \\
tai35b & (0.00, 0.01, 0.14) & 0.68 & (0.00, \textbf{0.00}, 0.00) & 0.35 & (0.14, 0.66, 1.41) & 0.52 & (27.12, 43.65, 49.07) & 13.82 & (0.00, \textbf{0.00}, 0.00) & 0.65 \\
tai40a & (0.07, \textbf{0.41}, 0.78) & 2.83 & (0.22, 0.54, 0.77) & 1.83 & (0.98, 1.55, 2.00) & 0.86 & (13.72, 15.58, 16.24) & 13.96 & (0.35, 0.57, 0.74) & 6.78 \\
tai40b & (0.00, \textbf{0.00}, 0.00) & 0.89 & (0.00, \textbf{0.00}, 0.00) & 0.32 & (0.04, 0.46, 2.71) & 0.80 & (42.25, 46.83, 55.28) & 13.96 & (0.00, \textbf{0.00}, 0.00) & 0.23 \\
tai50a & (0.64, 0.89, 1.07) & 5.56 & (0.64, \textbf{0.86}, 1.13) & 4.23 & (1.22, 1.63, 1.91) & 1.82 & (14.32, 15.64, 16.71) & 14.22 & (0.64, 0.88, 1.11) & 8.51 \\
tai50b & (0.00, 0.10, 0.37) & 4.63 & (0.00, 0.01, 0.05) & 2.32 & (0.20, 0.80, 1.74) & 1.92 & (43.10, 45.39, 49.12) & 14.22 & (0.00, \textbf{0.00}, 0.00) & 4.71 \\
tai60a & (0.85, 1.04, 1.32) & 9.65 & (0.72, 1.02, 1.29) & 6.84 & (1.41, 1.75, 2.11) & 2.58 & (14.67, 15.11, 15.42) & 17.97 & (0.50, \textbf{0.90}, 1.08) & 5.87 \\
tai60b & (0.00, 0.06, 0.47) & 8.68 & (0.00, 0.00, 0.02) & 3.86 & (0.11, 0.43, 0.68) & 2.73 & (41.85, 47.35, 51.30) & 17.94 & (0.00, \textbf{0.00}, 0.00) & 1.66 \\
tai64c & (0.00, \textbf{0.00}, 0.00) & 0.70 & (0.00, \textbf{0.00}, 0.00) & 0.07 & (0.24, 0.58, 1.15) & 0.21 & (217.55, 217.55, 217.55) & 18.08 & (0.00, \textbf{0.00}, 0.00) & 0.06 \\
tai80a & (0.83, 1.11, 1.23) & 23.31 & (0.95, 1.13, 1.22) & 13.17 & (1.44, 1.64, 1.86) & 4.40 & (13.27, 13.60, 13.88) & 18.74 & (0.83, \textbf{1.08}, 1.23) & 7.91 \\
tai80b & (0.00, 0.21, 0.67) & 23.25 & (0.05, 0.74, 1.37) & 8.56 & (0.42, 1.26, 1.92) & 6.02 & (37.65, 39.91, 42.42) & 18.73 & (0.00, \textbf{0.05}, 0.29) & 9.13 \\
tho150 & (0.08, \textbf{0.11}, 0.15) & 113.64 & (0.07, 0.13, 0.21) & 40.92 & (0.33, 0.51, 0.76) & 36.00 & (17.67, 18.09, 18.69) & 21.34 & (0.10, 0.13, 0.18) & 18.80 \\
tho30 & (0.00, \textbf{0.00}, 0.00) & 0.09 & (0.00, \textbf{0.00}, 0.00) & 0.16 & (0.00, 0.32, 0.57) & 0.54 & (22.27, 25.44, 29.29) & 13.67 & (0.00, \textbf{0.00}, 0.00) & 0.09 \\
tho40 & (0.00, 0.01, 0.05) & 2.49 & (0.00, 0.01, 0.05) & 1.13 & (0.14, 0.41, 0.89) & 0.84 & (22.31, 28.52, 30.38) & 13.93 & (0.00, \textbf{0.00}, 0.00) & 0.37 \\
wil100 & (0.01, \textbf{0.03}, 0.07) & 47.05 & (0.01, 0.03, 0.12) & 18.86 & (0.07, 0.16, 0.21) & 16.54 & (7.80, 8.17, 8.46) & 19.51 & (0.02, 0.03, 0.07) & 9.20 \\
wil50 & (0.00, 0.02, 0.04) & 5.07 & (0.00, \textbf{0.00}, 0.02) & 1.12 & (0.02, 0.07, 0.16) & 2.79 & (9.11, 9.86, 10.39) & 14.25 & (0.00, 0.01, 0.04) & 4.45 \\

\end{longtable}
}

{
\fontsize{8.8pt}{10.5pt}\selectfont
\setlength{\tabcolsep}{3pt}
\begin{longtable}{l cc cc cc cc}
\caption{Detailed Taixxeyy results. Gap is reported as (min, mean, max) over 10 runs; time is the average runtime (s).} \label{tab:taixxeyy_details} \\
\toprule
\textbf{Problem} & \multicolumn{2}{c}{\textbf{Ro-TS}} & \multicolumn{2}{c}{\textbf{BMA}} & \multicolumn{2}{c}{\textbf{IPFP}} & \multicolumn{2}{c}{\textbf{PLMA}} \\
\cmidrule(lr){2-3} \cmidrule(lr){4-5} \cmidrule(lr){6-7} \cmidrule(lr){8-9}
\textbf{Name} & Gap & Time & Gap & Time & Gap & Time & Gap & Time \\
\midrule
\endfirsthead

\multicolumn{9}{c}
{{\tablename\ \thetable{} -- continued from previous page}} \\
\toprule
\textbf{Problem} & \multicolumn{2}{c}{\textbf{Ro-TS}} & \multicolumn{2}{c}{\textbf{BMA}} & \multicolumn{2}{c}{\textbf{IPFP}} & \multicolumn{2}{c}{\textbf{PLMA}} \\
\cmidrule(lr){2-3} \cmidrule(lr){4-5} \cmidrule(lr){6-7} \cmidrule(lr){8-9}
\textbf{Name} & Gap & Time & Gap & Time & Gap & Time & Gap & Time \\
\midrule
\endhead

\bottomrule
\endfoot

tai27e01 & (0.00, 7.58, 19.16) & 0.62 & (0.00, \textbf{0.00}, 0.00) & 0.22 & (8.44, 15.86, 22.67) & 0.36 & (0.00, \textbf{0.00}, 0.00) & 0.11 \\
tai27e02 & (0.00, 0.25, 0.84) & 0.56 & (0.00, \textbf{0.00}, 0.00) & 0.15 & (6.25, 15.96, 26.53) & 0.33 & (0.00, \textbf{0.00}, 0.00) & 0.14 \\
tai27e03 & (0.00, 2.68, 14.18) & 0.66 & (0.00, \textbf{0.00}, 0.00) & 0.26 & (17.62, 24.44, 31.92) & 0.38 & (0.00, \textbf{0.00}, 0.00) & 0.10 \\
tai27e04 & (0.00, 89.99, 405.32) & 0.71 & (0.00, \textbf{0.00}, 0.00) & 0.33 & (12.40, 20.96, 33.88) & 0.41 & (0.00, \textbf{0.00}, 0.00) & 0.09 \\
tai27e05 & (2.15, 50.76, 433.70) & 0.82 & (0.00, \textbf{0.00}, 0.00) & 0.04 & (3.77, 17.15, 34.94) & 0.37 & (0.00, \textbf{0.00}, 0.00) & 0.07 \\
tai27e06 & (0.00, 40.48, 371.00) & 0.54 & (0.00, \textbf{0.00}, 0.00) & 0.12 & (5.69, 23.73, 35.18) & 0.39 & (0.00, \textbf{0.00}, 0.00) & 0.08 \\
tai27e07 & (0.00, 6.60, 14.88) & 0.60 & (0.00, \textbf{0.00}, 0.00) & 0.26 & (10.33, 22.32, 44.87) & 0.38 & (0.00, \textbf{0.00}, 0.00) & 0.07 \\
tai27e08 & (0.00, 147.98, 492.43) & 0.41 & (0.00, 1.68, 8.40) & 0.34 & (10.12, 31.15, 59.34) & 0.42 & (0.00, \textbf{0.00}, 0.00) & 0.06 \\
tai27e09 & (0.00, 69.25, 337.28) & 0.49 & (0.00, \textbf{0.00}, 0.00) & 0.10 & (3.31, 19.33, 33.49) & 0.39 & (0.00, \textbf{0.00}, 0.00) & 0.08 \\
tai27e10 & (0.00, 44.54, 438.41) & 0.36 & (0.00, \textbf{0.00}, 0.00) & 0.04 & (12.22, 23.15, 37.07) & 0.39 & (0.00, \textbf{0.00}, 0.00) & 0.08 \\
tai27e11 & (0.00, 41.68, 382.66) & 0.63 & (0.00, \textbf{0.00}, 0.00) & 0.15 & (1.38, 18.87, 27.80) & 0.40 & (0.00, \textbf{0.00}, 0.00) & 0.06 \\
tai27e12 & (0.00, 2.69, 11.40) & 0.67 & (0.00, 0.26, 1.30) & 0.34 & (2.02, 16.74, 31.47) & 0.35 & (0.00, \textbf{0.00}, 0.00) & 0.09 \\
tai27e13 & (0.00, 158.72, 390.63) & 0.59 & (0.00, \textbf{0.00}, 0.00) & 0.21 & (0.40, 13.43, 27.85) & 0.39 & (0.00, \textbf{0.00}, 0.00) & 0.05 \\
tai27e14 & (0.00, 3.42, 10.48) & 0.45 & (0.00, \textbf{0.00}, 0.00) & 0.01 & (4.65, 15.76, 32.40) & 0.39 & (0.00, \textbf{0.00}, 0.00) & 0.08 \\
tai27e15 & (0.00, 3.54, 16.36) & 0.67 & (0.00, \textbf{0.00}, 0.00) & 0.22 & (15.98, 21.71, 39.19) & 0.37 & (0.00, \textbf{0.00}, 0.00) & 0.09 \\
tai27e16 & (0.00, 34.65, 346.48) & 0.43 & (0.00, \textbf{0.00}, 0.00) & 0.20 & (0.00, 19.74, 31.31) & 0.38 & (0.00, \textbf{0.00}, 0.00) & 0.07 \\
tai27e17 & (0.00, 29.34, 293.39) & 0.26 & (0.00, \textbf{0.00}, 0.00) & 0.08 & (0.26, 16.54, 30.05) & 0.38 & (0.00, \textbf{0.00}, 0.00) & 0.06 \\
tai27e18 & (0.00, 85.84, 419.14) & 0.61 & (0.00, \textbf{0.00}, 0.00) & 0.10 & (0.00, 16.15, 25.89) & 0.37 & (0.00, \textbf{0.00}, 0.00) & 0.08 \\
tai27e19 & (0.00, 6.21, 15.27) & 0.76 & (0.00, 0.37, 3.66) & 0.34 & (4.61, 14.08, 31.98) & 0.37 & (0.00, \textbf{0.00}, 0.00) & 0.07 \\
tai27e20 & (0.00, 3.75, 8.49) & 0.51 & (0.00, \textbf{0.00}, 0.00) & 0.13 & (3.34, 22.24, 44.28) & 0.39 & (0.00, \textbf{0.00}, 0.00) & 0.06 \\
tai45e01 & (7.30, 58.87, 382.75) & 3.89 & (1.65, 9.95, 15.69) & 2.32 & (11.48, 24.05, 47.07) & 0.97 & (0.00, \textbf{0.00}, 0.00) & 0.15 \\
tai45e02 & (1.33, 60.63, 401.19) & 3.89 & (0.00, 5.93, 17.13) & 2.20 & (16.92, 24.29, 32.68) & 1.06 & (0.00, \textbf{0.00}, 0.00) & 0.21 \\
tai45e03 & (0.11, 202.98, 486.07) & 3.89 & (0.40, 9.18, 18.37) & 2.26 & (1.02, 19.10, 34.07) & 1.06 & (0.00, \textbf{0.00}, 0.00) & 0.19 \\
tai45e04 & (0.84, 136.56, 446.88) & 3.90 & (0.00, 6.89, 15.11) & 2.12 & (16.36, 30.26, 49.66) & 1.04 & (0.00, \textbf{0.00}, 0.00) & 0.31 \\
tai45e05 & (1.62, 181.32, 445.89) & 3.89 & (0.99, 6.92, 14.87) & 2.28 & (11.25, 28.66, 44.38) & 1.06 & (0.00, \textbf{0.00}, 0.00) & 0.21 \\
tai45e06 & (0.51, 56.11, 467.88) & 3.90 & (2.84, 15.35, 25.47) & 2.19 & (6.05, 16.23, 29.89) & 1.03 & (0.00, \textbf{0.00}, 0.00) & 0.14 \\
tai45e07 & (0.43, 72.06, 347.09) & 3.90 & (0.00, 7.24, 24.18) & 2.10 & (2.07, 21.99, 37.52) & 1.01 & (0.00, \textbf{0.00}, 0.00) & 0.21 \\
tai45e08 & (1.43, 156.65, 371.93) & 3.90 & (0.00, 10.09, 17.73) & 2.11 & (2.93, 22.55, 37.90) & 1.01 & (0.00, \textbf{0.00}, 0.00) & 0.15 \\
tai45e09 & (1.87, 102.48, 486.58) & 3.89 & (2.65, 12.59, 22.32) & 2.23 & (7.04, 22.74, 36.19) & 0.99 & (0.00, \textbf{0.00}, 0.00) & 0.15 \\
tai45e10 & (0.02, 154.62, 380.81) & 3.89 & (0.00, 5.08, 15.66) & 2.35 & (1.83, 17.21, 32.46) & 1.00 & (0.00, \textbf{0.00}, 0.00) & 0.20 \\
tai45e11 & (0.00, 51.49, 425.75) & 3.74 & (1.66, 10.90, 20.61) & 2.26 & (12.50, 20.74, 37.36) & 1.12 & (0.00, \textbf{0.00}, 0.00) & 0.14 \\
tai45e12 & (0.00, 51.22, 414.57) & 3.68 & (0.00, 7.14, 16.22) & 2.13 & (9.19, 19.29, 38.46) & 0.99 & (0.00, \textbf{0.00}, 0.00) & 0.18 \\
tai45e13 & (0.00, 85.56, 381.99) & 3.60 & (0.88, 7.06, 17.81) & 2.26 & (10.75, 24.52, 43.82) & 1.03 & (0.00, \textbf{0.00}, 0.00) & 0.18 \\
tai45e14 & (0.00, 19.08, 87.45) & 3.80 & (0.00, 7.14, 17.60) & 2.39 & (16.25, 25.77, 36.24) & 1.07 & (0.00, \textbf{0.00}, 0.00) & 0.20 \\
tai45e15 & (0.00, 126.79, 401.33) & 3.86 & (0.62, 9.06, 14.96) & 2.35 & (1.43, 21.44, 31.57) & 1.09 & (0.00, \textbf{0.00}, 0.00) & 0.14 \\
tai45e16 & (0.37, 42.02, 374.85) & 3.89 & (0.00, 4.46, 15.31) & 2.16 & (0.95, 19.01, 42.52) & 1.00 & (0.00, \textbf{0.00}, 0.00) & 0.13 \\
tai45e17 & (2.18, 98.03, 300.45) & 3.89 & (0.00, 5.70, 10.83) & 2.09 & (0.23, 12.26, 25.89) & 1.02 & (0.00, \textbf{0.00}, 0.00) & 0.19 \\
tai45e18 & (0.00, 87.40, 413.12) & 3.58 & (0.00, 6.08, 16.62) & 1.85 & (7.96, 21.26, 33.48) & 1.02 & (0.00, \textbf{0.00}, 0.00) & 0.13 \\
tai45e19 & (0.00, 172.91, 429.12) & 3.64 & (0.00, 8.60, 14.87) & 2.18 & (18.38, 24.73, 33.42) & 1.02 & (0.00, \textbf{0.00}, 0.00) & 0.19 \\
tai45e20 & (2.09, 121.06, 566.27) & 3.89 & (3.01, 12.34, 21.29) & 2.33 & (10.51, 24.10, 35.02) & 1.00 & (0.00, \textbf{0.00}, 0.00) & 0.16 \\
tai75e01 & (9.73, 95.31, 283.74) & 18.47 & (8.23, 12.82, 16.08) & 11.52 & (20.49, 29.09, 36.29) & 2.60 & (0.00, \textbf{0.00}, 0.00) & 1.38 \\
tai75e02 & (11.80, 99.41, 294.86) & 18.50 & (9.06, 15.79, 22.51) & 11.30 & (24.92, 31.31, 42.26) & 2.53 & (0.00, \textbf{0.00}, 0.00) & 2.29 \\
tai75e03 & (0.81, 67.62, 272.54) & 18.48 & (2.44, 15.23, 21.96) & 10.91 & (20.15, 28.83, 39.44) & 2.59 & (0.00, \textbf{0.00}, 0.00) & 1.09 \\
tai75e04 & (2.66, 145.42, 285.63) & 18.49 & (1.29, 13.45, 20.29) & 11.23 & (19.00, 27.71, 37.71) & 2.58 & (0.00, \textbf{0.01}, 0.15) & 1.14 \\
tai75e05 & (3.15, 89.91, 278.31) & 18.49 & (2.86, 16.30, 24.08) & 11.00 & (17.11, 24.36, 35.36) & 2.40 & (0.00, \textbf{0.00}, 0.00) & 1.11 \\
tai75e06 & (3.32, 48.85, 311.87) & 18.48 & (3.59, 20.30, 30.60) & 11.39 & (23.17, 40.10, 47.44) & 2.73 & (0.00, \textbf{0.00}, 0.00) & 0.77 \\
tai75e07 & (8.78, 196.51, 324.15) & 18.48 & (5.08, 10.60, 15.53) & 11.19 & (19.78, 26.45, 36.98) & 2.38 & (0.00, \textbf{0.55}, 2.08) & 3.07 \\
tai75e08 & (10.96, 76.87, 293.53) & 18.49 & (6.75, 13.48, 19.24) & 11.04 & (14.51, 26.28, 36.46) & 2.27 & (0.00, \textbf{0.00}, 0.00) & 1.26 \\
tai75e09 & (1.58, 113.32, 275.30) & 18.48 & (1.79, 14.09, 22.13) & 11.31 & (19.83, 26.90, 33.08) & 2.56 & (0.00, \textbf{0.00}, 0.00) & 1.05 \\
tai75e10 & (5.90, 149.17, 288.37) & 18.47 & (7.64, 14.54, 19.48) & 11.20 & (14.01, 22.03, 34.58) & 2.47 & (0.00, \textbf{0.00}, 0.00) & 0.90 \\
tai75e11 & (3.27, 124.39, 294.39) & 18.47 & (13.39, 18.60, 24.81) & 11.06 & (13.73, 26.45, 37.61) & 2.57 & (0.00, \textbf{0.00}, 0.00) & 0.84 \\
tai75e12 & (12.48, 154.06, 295.78) & 18.47 & (6.60, 13.80, 17.95) & 11.39 & (25.02, 30.11, 34.25) & 2.41 & (0.00, \textbf{0.20}, 2.04) & 1.89 \\
tai75e13 & (5.99, 90.61, 282.66) & 18.48 & (5.62, 13.19, 22.99) & 11.10 & (11.39, 25.02, 38.85) & 2.56 & (0.00, \textbf{0.31}, 3.15) & 1.60 \\
tai75e14 & (4.25, 19.70, 49.03) & 18.48 & (8.89, 16.01, 22.82) & 10.92 & (20.41, 29.04, 36.85) & 2.57 & (0.00, \textbf{0.00}, 0.00) & 1.21 \\
tai75e15 & (3.47, 145.10, 286.22) & 18.50 & (3.79, 15.90, 27.65) & 11.32 & (17.24, 27.54, 33.59) & 2.59 & (0.00, \textbf{0.00}, 0.00) & 0.65 \\
tai75e16 & (19.85, 207.84, 292.96) & 18.50 & (11.17, 16.68, 22.50) & 11.14 & (16.56, 22.94, 30.36) & 2.44 & (0.00, \textbf{0.00}, 0.00) & 0.93 \\
tai75e17 & (2.29, 140.05, 272.35) & 18.47 & (3.91, 17.74, 29.86) & 11.12 & (7.93, 27.85, 42.62) & 2.63 & (0.00, \textbf{0.00}, 0.00) & 0.67 \\
tai75e18 & (8.33, 70.83, 296.13) & 18.48 & (7.23, 15.67, 18.74) & 11.40 & (17.60, 27.43, 35.41) & 2.58 & (0.00, \textbf{0.51}, 5.14) & 1.47 \\
tai75e19 & (4.38, 148.63, 332.59) & 18.51 & (1.87, 17.60, 24.98) & 11.37 & (16.17, 29.39, 36.27) & 2.34 & (0.00, \textbf{0.00}, 0.00) & 0.91 \\
tai75e20 & (1.11, 41.99, 289.84) & 18.51 & (4.17, 16.33, 30.98) & 11.19 & (12.16, 23.94, 30.26) & 2.50 & (0.00, \textbf{0.00}, 0.00) & 1.13 \\
tai125e01 & (11.43, 45.68, 298.57) & 72.52 & (13.92, 20.44, 24.68) & 66.04 & (17.79, 27.42, 33.23) & 9.35 & (2.03, \textbf{3.45}, 5.79) & 8.79 \\
tai125e02 & (6.92, 106.10, 318.56) & 72.58 & (11.93, 17.76, 21.82) & 64.02 & (21.58, 25.23, 30.64) & 9.32 & (0.53, \textbf{3.15}, 5.66) & 8.79 \\
tai125e03 & (8.68, 76.74, 308.06) & 72.48 & (20.91, 23.69, 27.42) & 64.91 & (21.30, 29.87, 36.27) & 9.19 & (1.61, \textbf{4.47}, 8.73) & 8.71 \\
tai125e04 & (1.94, 11.68, 20.43) & 72.53 & (9.85, 15.76, 23.32) & 66.49 & (23.30, 29.17, 38.79) & 9.47 & (-1.20, \textbf{1.06}, 4.71) & 7.03 \\
tai125e05 & (-0.04, 60.94, 272.41) & 72.41 & (1.79, 10.53, 19.84) & 65.81 & (19.70, 24.75, 31.86) & 8.00 & (0.04, \textbf{2.50}, 4.79) & 8.55 \\
tai125e06 & (11.51, 98.77, 292.30) & 72.50 & (4.97, 15.97, 22.21) & 65.33 & (19.27, 24.05, 28.30) & 7.90 & (2.06, \textbf{5.64}, 8.87) & 8.73 \\
tai125e07 & (10.22, 46.77, 284.53) & 72.45 & (12.03, 16.52, 19.18) & 66.42 & (21.61, 30.61, 34.91) & 8.06 & (-1.80, \textbf{3.32}, 8.56) & 7.84 \\
tai125e08 & (9.30, 112.84, 265.66) & 72.50 & (9.10, 13.68, 18.88) & 65.52 & (16.88, 24.99, 31.88) & 8.28 & (1.17, \textbf{4.56}, 7.13) & 8.65 \\
tai125e09 & (10.25, 72.13, 300.55) & 72.52 & (14.03, 17.70, 21.53) & 66.72 & (23.70, 27.89, 33.86) & 9.02 & (1.21, \textbf{4.10}, 7.64) & 8.87 \\
tai125e10 & (11.51, 106.42, 302.24) & 72.59 & (16.21, 18.40, 21.61) & 65.67 & (20.32, 27.32, 33.29) & 8.48 & (2.05, \textbf{3.56}, 4.88) & 8.97 \\
tai125e11 & (4.66, 118.39, 284.63) & 72.60 & (11.38, 14.64, 16.93) & 65.85 & (19.76, 27.42, 32.06) & 8.38 & (3.07, \textbf{6.14}, 7.77) & 8.93 \\
tai125e12 & (16.09, 111.32, 317.21) & 72.53 & (16.92, 22.47, 26.50) & 63.78 & (23.71, 29.60, 38.10) & 8.02 & (2.26, \textbf{4.70}, 7.66) & 9.02 \\
tai125e13 & (6.73, 117.73, 273.36) & 72.61 & (7.90, 15.69, 25.32) & 66.33 & (24.20, 28.02, 34.10) & 7.82 & (2.62, \textbf{3.39}, 4.41) & 8.91 \\
tai125e14 & (7.23, 100.60, 299.06) & 72.55 & (16.03, 20.46, 25.34) & 69.27 & (22.97, 32.20, 36.79) & 9.16 & (1.20, \textbf{3.99}, 7.13) & 9.14 \\
tai125e15 & (5.90, 125.11, 300.54) & 72.50 & (4.03, 13.48, 18.91) & 65.44 & (19.74, 25.21, 29.43) & 8.47 & (-1.13, \textbf{2.35}, 5.75) & 7.53 \\
tai125e16 & (3.94, 114.62, 269.72) & 72.52 & (3.82, 9.79, 16.49) & 66.36 & (12.41, 18.97, 23.78) & 8.21 & (-2.52, \textbf{3.78}, 6.14) & 8.69 \\
tai125e17 & (3.49, 17.88, 34.90) & 72.49 & (7.96, 17.13, 24.73) & 64.71 & (23.08, 26.79, 31.88) & 8.30 & (1.51, \textbf{2.42}, 3.30) & 9.03 \\
tai125e18 & (5.01, 84.94, 264.67) & 72.54 & (3.64, 10.13, 14.62) & 66.25 & (16.86, 22.36, 29.61) & 8.17 & (0.26, \textbf{2.13}, 3.69) & 9.04 \\
tai125e19 & (10.28, 74.85, 301.27) & 72.51 & (13.41, 18.36, 22.80) & 68.34 & (26.38, 30.59, 34.97) & 8.73 & (0.84, \textbf{5.34}, 8.76) & 8.95 \\
tai125e20 & (8.04, 47.04, 302.06) & 72.50 & (11.40, 17.61, 22.33) & 67.02 & (18.23, 26.13, 29.70) & 8.10 & (-0.23, \textbf{2.96}, 5.02) & 8.83 \\
tai175e01 & (8.04, 136.55, 265.64) & 158.69 & (18.56, 21.96, 26.15) & 195.71 & (19.22, 24.06, 32.44) & 17.62 & (6.76, \textbf{8.69}, 12.14) & 14.34 \\
tai175e02 & (9.88, 47.83, 277.88) & 158.74 & (20.58, 24.50, 30.61) & 194.60 & (18.68, 25.77, 30.26) & 16.58 & (9.31, \textbf{11.76}, 16.08) & 14.48 \\
tai175e03 & (15.32, 164.24, 305.14) & 158.76 & (18.61, inf, inf) & 200.77 & (23.00, 27.93, 32.75) & 17.55 & (6.30, \textbf{8.17}, 10.21) & 14.18 \\
tai175e04 & (7.00, 60.56, 248.45) & 158.69 & (16.85, 19.66, 25.17) & 188.51 & (9.38, 17.93, 22.35) & 17.09 & (4.55, \textbf{7.00}, 9.01) & 14.42 \\
tai175e05 & (13.93, 52.22, 297.51) & 158.71 & (16.93, 24.48, 33.26) & 193.30 & (24.95, 27.94, 31.82) & 16.35 & (7.83, \textbf{9.97}, 13.27) & 14.45 \\
tai175e06 & (11.32, 42.79, 272.55) & 158.70 & (19.83, 25.42, 29.35) & 189.75 & (21.06, 25.38, 30.60) & 16.68 & (7.42, \textbf{9.21}, 10.47) & 14.37 \\
tai175e07 & (11.09, 71.72, 279.85) & 158.75 & (15.87, 25.24, 29.30) & 196.00 & (22.78, 25.49, 28.59) & 16.58 & (6.54, \textbf{10.12}, 13.34) & 14.48 \\
tai175e08 & (6.15, 40.32, 262.51) & 158.65 & (11.93, 19.34, 24.77) & 195.71 & (16.25, 22.98, 26.68) & 16.51 & (9.31, \textbf{12.14}, 15.28) & 14.62 \\
tai175e09 & (-0.03, 57.54, 248.44) & 147.71 & (10.96, 17.19, 21.14) & 192.34 & (13.23, 18.37, 22.17) & 16.75 & (11.52, \textbf{13.28}, 16.97) & 14.33 \\
tai175e10 & (7.47, 49.93, 270.09) & 158.71 & (19.37, 23.77, 30.48) & 192.17 & (17.61, 24.59, 28.04) & 17.05 & (5.84, \textbf{11.09}, 15.47) & 14.06 \\
tai175e11 & (10.39, 93.45, 270.99) & 158.84 & (16.66, 21.36, 24.66) & 193.03 & (17.87, 22.62, 29.26) & 16.59 & (4.64, \textbf{11.05}, 17.63) & 14.78 \\
tai175e12 & (11.51, 42.96, 263.97) & 158.79 & (11.97, inf, inf) & 201.86 & (11.86, 20.83, 27.91) & 16.80 & (3.17, \textbf{5.32}, 7.09) & 14.65 \\
tai175e13 & (10.38, 17.83, 29.28) & 158.69 & (11.08, 18.74, 23.28) & 197.75 & (18.63, 22.93, 27.17) & 16.71 & (4.60, \textbf{8.12}, 10.55) & 14.14 \\
tai175e14 & (2.43, 83.00, 255.31) & 158.63 & (10.31, 19.67, 27.48) & 196.02 & (9.50, 20.07, 27.54) & 16.71 & (3.97, \textbf{7.94}, 10.15) & 14.41 \\
tai175e15 & (13.36, 43.07, 274.57) & 158.77 & (10.55, 17.51, 25.60) & 195.31 & (18.11, 22.42, 30.63) & 17.21 & (6.61, \textbf{9.56}, 12.55) & 14.37 \\
tai175e16 & (7.68, 106.06, 311.13) & 158.80 & (15.71, inf, inf) & 195.51 & (19.09, 25.37, 30.55) & 16.85 & (5.09, \textbf{7.63}, 10.25) & 14.68 \\
tai175e17 & (8.60, 65.70, 266.28) & 158.71 & (14.75, 20.46, 25.81) & 185.47 & (18.13, 21.64, 23.99) & 16.53 & (3.66, \textbf{5.47}, 7.82) & 14.42 \\
tai175e18 & (12.11, 42.13, 266.07) & 158.66 & (10.61, 13.94, 16.99) & 189.47 & (19.35, 22.05, 24.82) & 17.27 & (4.13, \textbf{9.35}, 11.77) & 14.56 \\
tai175e19 & (6.29, 95.24, 271.39) & 158.82 & (16.01, 20.63, 31.39) & 193.72 & (17.88, 23.23, 27.26) & 16.66 & (3.78, \textbf{7.91}, 11.83) & 14.52 \\
tai175e20 & (9.31, 44.14, 282.49) & 158.73 & (15.96, 21.38, 27.30) & 196.64 & (17.47, 24.71, 27.46) & 17.13 & (4.11, \textbf{8.02}, 12.96) & 14.31 \\

\end{longtable}
}

\end{document}